\documentclass[twoside,11pt]{article}
\usepackage{amsmath}
\usepackage{amssymb}
\usepackage{jmlr2e}
\usepackage{multirow}
\usepackage{rotating}
\graphicspath{{./}{./figures/}}

 \usepackage{algorithm,algorithmic}
 \usepackage[tight,footnotesize]{subfigure}

 \newcommand{\RR}{\mathbb{R}}

 \newcommand{\T}{{}^\top}

 \newcommand{\bm}[1]{\boldsymbol{#1}}

 \newcommand{\vq}{\boldsymbol{q}}
 \newcommand{\vz}{\boldsymbol{z}}
 \newcommand{\vx}{\boldsymbol{x}}
 \newcommand{\vy}{\boldsymbol{y}}
 \newcommand{\vw}{\boldsymbol{w}}
 
 \newcommand{\vv}{\boldsymbol{v}}
 \newcommand{\valpha}{\boldsymbol{\alpha}}
 \newcommand{\vbeta}{\boldsymbol{\beta}}
 \newcommand{\vdelta}{\boldsymbol{\delta}}

 \newcommand{\vxi}{\boldsymbol{\xi}}

 \newcommand{\mA}{\boldsymbol{A}}
 
 \newcommand{\mS}{\boldsymbol{S}}
 
 \newcommand{\mI}{\boldsymbol{I}}
 \newcommand{\mU}{\boldsymbol{U}}
 \newcommand{\mV}{\boldsymbol{V}}

 \newcommand{\gf}{{\mathfrak{g}}}
 \newcommand{\Gf}{{\mathfrak{G}}}

 \def\dot#1#2{\left\langle #1,#2\right\rangle}

 \newcommand{\minimize}{\mathop{\rm minimize}}
 \newcommand{\maximize}{\mathop{\rm maximize}}
 \newcommand{\argmin}{\mathop{\rm argmin}}
 
 \newcommand{\subjectto}{\mbox{\rm subject to}}

 \newcommand{\Eqref}[1]{Eq.~{\eqref{#1}}}
 \newcommand{\Eqsref}[1]{Eqs.~{\eqref{#1}}}
 \newcommand{\Secref}[1]{Section~{\ref{#1}}}
 \newcommand{\Figref}[1]{Figure~{\ref{#1}}}
 \newcommand{\Tabref}[1]{Table~{\ref{#1}}}

\newcommand{\ST}[1]{{\rm prox}_{#1}^{\ell_1}}
\newcommand{\prox}[1]{{\rm prox}_{#1}}

 \title{Super-Linear Convergence of Dual Augmented Lagrangian Algorithm
 for Sparsity Regularized Estimation}

\author{\name Ryota Tomioka \email tomioka@mist.i.u-tokyo.ac.jp \\
       \addr  Department of Mathematical Informatics,\\
        The University of Tokyo,\\
	7-3-1, Hongo, Bunkyo-ku, Tokyo, 113-8656, Japan.
       \AND
       \name Taiji Suzuki \email s-taiji@stat.t.u-tokyo.ac.jp\\
       \addr  Department of Mathematical Informatics,\\
        The University of Tokyo,\\
	7-3-1, Hongo, Bunkyo-ku, Tokyo, 113-8656, Japan.
	\AND
       \name Masashi Sugiyama \email sugi@cs.titech.ac.jp \\
       \addr Department of Computer Science,\\
       Tokyo Institute of Technology,\\
       2-12-1-W8-74, O-okayama, Meguro-ku, Tokyo, 152-8552, Japan.}

\editor{}

\jmlrheading{1}{2000}{1-48}{4/00}{10/00}{Ryota Tomioka, Taiji Suzuki,
and Masashi Sugiyama}

\ShortHeadings{Dual Augmented-Lagrangian Converges Super-Linearly}{Tomioka, Suzuki,
and Sugiyama}
\firstpageno{1}
 
\begin{document}
 \maketitle

\begin{abstract}
We analyze the convergence behaviour of a recently proposed
 algorithm for regularized estimation called Dual Augmented Lagrangian
 (DAL).
Our analysis is based on a new interpretation of DAL as a proximal
 minimization algorithm.
We theoretically show
 under some conditions that DAL converges super-linearly in a
 non-asymptotic and global sense. Due to a special modelling of sparse
 estimation problems in the context of machine learning,
 the assumptions we make are milder and more natural than those made in
 conventional analysis of augmented Lagrangian algorithms.
In addition, the new interpretation enables us to generalize DAL to wide
 varieties of sparse estimation problems.
 We  experimentally confirm our analysis in  a large scale
 $\ell_1$-regularized logistic regression problem and   extensively
 compare the efficiency of DAL algorithm to previously  proposed
 algorithms on both synthetic and benchmark datasets. 
\end{abstract}
\begin{description}
 \item[Running title:] Dual Augmented-Lagrangian Converges Super-Linearly

 \item[Keywords:] Dual Augmented Lagrangian, Proximal Minimization,
	    Global Convergence, Sparse Estimation, Convex Optimization
\end{description}

\section{Introduction}
Sparse estimation through convex regularization has become a common practice
in many application areas including bioinformatics and natural language
processing. However facing the rapid increase in the size of data-sets
that we analyze everyday, clearly needed is the development of
optimization algorithms that are tailored for machine learning
applications.

Regularization-based sparse estimation methods estimate unknown variables through the
minimization of a loss term (or a data-fit term) plus a regularization
term. In this paper, we focus on convex methods; i.e., both the loss term
and the regularization term are convex functions of unknown variables.
Regularizers may be non-differentiable on some points; the
non-differentiability can promote various types of sparsity on the solution.

Although the problem is convex, there are three factors that challenge the
straight-forward application of general tools for convex
optimization~\citep{BoydBook} in 
the context of machine learning.

 The first factor is the diversity of loss
functions. Arguably the squared loss is most commonly used in the
field of signal/image reconstruction, in which many algorithms for
sparse estimation have been
developed~\citep{FigNow03,DauDefMol04,CaiCanShe08}. However the variety
of loss functions is much wider in machine learning, to name a few,
logistic loss and other log-linear loss functions. Note that these
functions are not necessarily strongly convex like the squared loss. See
Table~\ref{tab:loss} for a list of loss functions that we consider.

The second factor is the nature of the data matrix, which we call the design
matrix in this paper. For a regression problem, the design matrix is
defined by stacking input vectors along rows. If the input
vectors are numerical (e.g., gene expression data), the design matrix is
dense and has no structure. In addition, the characteristics of the
matrix (e.g., the condition number) is unknown until the data is
provided. Therefore, we would like to minimize assumptions about the design
matrix, such as, sparse, structured, or well conditioned.

The third factor is the large number of unknown variables (or parameters)
compared to observations. This is a situation regularized estimation methods
are commonly applied. This factor may have been overlooked in the
context of signal denoising, in which the number of observations and the
number of parameters are equal.

Various methods have been proposed for efficient sparse estimation
(see \cite{FigNow03,DauDefMol04,ComWaj05,AndGao07,KohKimBoy07,WriNowFig09,BecTeb09,YuVisGueSch10},
and the references therein). Many previous studies focus on the {\em
non-differentiability} of the regularization term. In contrast, we focus
on the {\em couplings} between variables (or non-separability) caused by
the design matrix. In fact, if the optimization problem can be decomposed
into smaller (e.g., containing a single variable) problems, optimization
is easy.
Recently \cite{WriNowFig09} showed that the so called
iterative shrinkage/thresholding (IST) method (see
\cite{FigNow03,DauDefMol04,ComWaj05,FigBioNow07}) can be seen  as
an iterative {\em separable approximation} process. 

In this paper, we show that a recently proposed dual augmented Lagrangian
(DAL) algorithm~\citep{TomSug09} can be considered as an {\em exact}
(up to finite tolerance) version of the iterative approximation process
discussed in \cite{WriNowFig09}.
Our formulation is based on the connection between the proximal
minimization~\citep{Roc76a} and the augmented Lagrangian (AL)
algorithm~\citep{Hes69,Pow69,Roc76b,Ber82}.
The proximal minimization framework also allows us to rigorously study
the convergence behaviour of DAL. We show that DAL converges
super-linearly under some mild conditions, 
which means that the number of iterations that we need to obtain an
$\epsilon$-accurate solution grows no greater than logarithmically with
$1/\epsilon$. 
Due to the generality of the framework, our analysis applies to a wide
variety of practically important regularizers.
Our analysis improves the classical result on the convergence of
augmented Lagrangian algorithms in \cite{Roc76b} by taking special 
structures of sparse estimation into account. In addition, we make no
asymptotic arguments as in \cite{Roc76b} and \cite{KorBer76}; instead our
convergence analysis is build on top of the recent result in 
\cite{BecTeb09}.

Augmented Lagrangian formulations have also been considered in
\cite{YinOshGolDar08} and \cite{GolOsh08} for sparse
signal reconstruction. What differentiates DAL approach of
\cite{TomSug09} from those studied earlier is that the AL algorithm is
applied to the dual problem (see \Secref{sec:dalreview}), which results
in an inner minimization problem that can be solved efficiently
exploiting the sparsity of intermediate solutions (see \Secref{sec:dall1}). Applying AL
formulation to the dual problem also plays an important role in the
convergence analysis because some loss functions (e.g., logistic loss)
are not strongly convex in the primal; see
\Secref{sec:analysis}. Recently \cite{YanZha09} compared primal and dual
augmented Lagrangian algorithms for $\ell_1$-problems and reported that
the dual formulation was more efficient. See also \cite{TomSuzSug11} for
related discussions.


%

 This paper is organized as follows.
 In \Secref{sec:framework}, we  mathematically formulate the sparse
 estimation problem and we review DAL algorithm.
 We derive DAL algorithm from the
 proximal  minimization framework in  \Secref{sec:proximal_view}, and  
discuss special instances of DAL algorithm are discussed in \Secref{sec:instances}.
 In \Secref{sec:analysis}, we theoretically analyze the convergence
behaviour of DAL algorithm.
 We discuss previously proposed algorithms in \Secref{sec:algorithms}
 and contrast them with DAL.
 In \Secref{sec:results} we confirm our
analysis in a simulated $\ell_1$-regularized logistic regression
problem. Moreover, we extensively compare recently proposed algorithms
for $\ell_1$-regularized logistic regression including DAL in  
synthetic and benchmark datasets under a variety of conditions.
Finally we conclude the paper in \Secref{sec:summary}. Most of the
proofs are given in the appendix.

\section{Sparse estimation problem and DAL algorithm}
\label{sec:framework}
In this section, we first formulate the sparse estimation problem as a
convex optimization problem, and state our assumptions. Next we derive
DAL algorithm for $\ell_1$-problem as an augmented Lagrangian method in
the dual.

\subsection{Objective}
We consider the problem of estimating an $n$ dimensional parameter
vector from $m$ training examples as described in the following
optimization problem: 
 \begin{align}
 \label{eq:problem}
  \minimize_{\vw\in\RR^n}\quad \underbrace{f_{\ell}(\mA\vw)+\phi_{\lambda}(\vw)}_{=:f(\vw)},
 \end{align}
 where $\vw\in\RR^n$ is the parameter vector to be estimated,
 $\mA\in\RR^{m\times n}$ is a design matrix, and $f_\ell(\cdot)$ is a
 loss function. We call the first term in the minimand the loss term and
 the second term the regularization term, or the regularizer.

We assume that the loss function
 $f_\ell:\,\RR^m\rightarrow\RR\cup\{+\infty\}$ is a closed proper
 strictly convex
 function\footnote{``Closed'' means that the epigraph
 $\{(\vz,y)\in\RR^{m+1}:y\geq f_{\ell}(\vz)\}$ is a closed set, and ``proper'' means that the function is not
 everywhere $+\infty$; see e.g., \cite{Roc70}. In the sequel, we use the word ``convex
 function'' in the meaning of ``closed proper convex function''.}. See
 Table~\ref{tab:loss} for examples of loss functions. We
 assume that $f_{\ell}$ has Lipschitz continuous gradient with modulus $1/\gamma$ (see Assumption {\bf
 (A2)} in \Secref{sec:approx}).
If $f_{\ell}$ is twice differentiable, this condition is equivalent to
saying that the maximum eigenvalue of the Hessian of $f_{\ell}$ is
uniformly bounded by $1/\gamma$. Such $\gamma$ exists for example for
quadratic loss, logistic loss, and other log-linear losses. However,
non-smooth loss functions (e.g., the hinge loss and the absolute loss)
are excluded. Note that since we separate the data matrix $\mA$ from the
loss function, we can quantify the above constant $\gamma$ without
examining the data.
Moreover, we assume that the convex conjugate\footnote{The convex
conjugate of a function 
$f:\,\RR^n\rightarrow\RR\cup\{+\infty\}$ is a function $f^\ast$ over
$\RR^n$ that takes values in $\RR^n\cup\{+\infty\}$ and is defined as
$f^\ast(\vy)=\sup_{\vx\in\RR^n}(\vy\T\vx-f(\vx))$.} $f_{\ell}^\ast$ is
(essentially) twice  differentiable. Note that the first order
differentiability of the  convex conjugate $f_{\ell}^{\ast}$ is implied
by the strict convexity of the loss function
$f_{\ell}$~\citep[Theorem~26.3]{Roc70}.

The regularization term $\phi_{\lambda}(\vw)$ is a convex possibly
 non-differentiable function. In addition, we assume 
that for all $\eta>0$, $\eta\phi_{\lambda}(\vw)=\phi_{\eta\lambda}(\vw)$. 

An important special case, which has been studied by many authors~\citep{Tib96,EfrHasTibJoh04,AndGao07,KohKimBoy07} is the
 $\ell_1$-regularization:
\begin{align}
 \label{eq:problemL1}
  \minimize_{\vw\in\RR^n}\quad f_{\ell}(\mA\vw)+\lambda\|\vw\|_1,
\end{align} 
where $\|\vw\|_1=\sum_{j=1}^n|w_j|$ is the $\ell_1$-norm of $\vw$.

\subsection{Dual augmented Lagrangian (DAL) algorithm}
\label{sec:dalreview}
In this subsection, we review DAL algorithm following the line of
\cite{TomSug09}. Although, the squared
loss function and the $\ell_1$-regularizer were considered in the
original paper, we deal with a slightly more general setting in
\Eqref{eq:problemL1} for notational convenience; i.e., we consider
general closed convex loss functions instead of the squared loss.
For general information on augmented Lagrangian
algorithms~\citep{Pow69,Hes69,Roc76b}, see \cite{Ber82} and \cite{NocWri99}.


Let $\phi_{\lambda}(\vw)$ be the $\ell_1$-regularizer, i.e., $\phi_{\lambda}(\vw)=\lambda\|\vw\|_1=\lambda\sum_{j=1}^n|w_j|$.
Using the Fenchel duality theorem~\citep{Roc70}, the dual of the
problem~\eqref{eq:problemL1} can be written as follows: 
\begin{align}
\label{eq:dual_prob}
\maximize_{\valpha\in\RR^m,\vv\in\RR^n} \qquad & -f_{\ell}^{\ast}(-\valpha)-\delta_\lambda^\infty(\vv),\\
\label{eq:dual_const}
\subjectto \qquad & \vv=\mA\T\valpha,
\end{align}
where $\delta_\lambda^\infty$ is the indicator function~\cite[p28]{Roc70} of the
$\ell_{\infty}$-ball of radius $\lambda$, namely
\begin{align}
\label{eq:indicatorL1}
 \delta_{\lambda}^\infty(\vv)&=\sum_{j=1}^n\delta_{\lambda}^{\infty}(v_j),
\end{align}
where $ \delta_{\lambda}^\infty(v_j)= 0$, if $|v_j|\leq\lambda$,
and $+\infty$ otherwise.

Let us consider the augmented Lagrangian (AL) function $L_{\eta}$ with
respect to the above dual problem~\eqref{eq:dual_prob}
\begin{align}
\label{eq:ALfunc_basic}
 L_{\eta}(\valpha,\vv;\vw)=-f_{\ell}^\ast(-\valpha)-\delta_{\lambda}^{\infty}(\vv)+\vw\T(\vv-\mA\T\valpha)-\frac{\eta}{2}\|\vv-\mA\T\valpha\|^2,
\end{align}
where the primal variable $\vw\in\RR^{n}$ is interpreted as a Lagrangian multiplier
vector in the AL framework. Note that the AL function is the ordinary Lagrangian
if $\eta=0$.

Let $\eta_0,\eta_1,\ldots$ be a non-decreasing sequence of positive
numbers. At every time step $t$, given the current primal solution $\vw^t$, 
we maximize the AL function
$L_{\eta_t}(\valpha,\vv;\vw^t)$ with respect to $\valpha$ and $\vv$.
The maximizer $(\valpha^t,\vv^t)$ is used to update the primal solution
(Lagrangian multiplier) $\vw^{t}$ as follows:
\begin{align}
\label{eq:dal_update_basic}
\vw^{t+1}&=\vw^{t}+\eta_t(\mA\T\valpha^t-\vv^t).
\end{align}

Note that the maximization of the AL function \eqref{eq:ALfunc_basic}
with respect to $\vv$ can be carried out in a closed form, because
the terms involved in the maximization can be separated into $n$
terms, each containing single $v_j$, as follows:
\begin{align*}
 L_{\eta_t}(\valpha,\vv)=-f_{\ell}^\ast(-\valpha)-\sum_{j=1}^n\left(\frac{\eta_t}{2}(v_j-(\vw^{t}/\eta_{t}+\mA\T\valpha)_j)^2+
\delta_\lambda^{\infty}(v_j)\right),
\end{align*} 
where $(\cdot)_j$ denotes the $j$th element of a vector. Since
$\delta_\lambda^{\infty}(v_j)$ is infinity outside the domain
$-\lambda\leq v_j\leq \lambda$, the maximizer $\vv^t(\valpha)$ is
obtained as a projection onto the $\ell_\infty$ ball of radius
$\lambda$ as follows (see also \Figref{fig:envelopes}):
\begin{align}
\label{eq:clipping}
 \vv^t(\valpha)&={\rm proj}_{[-\lambda,\lambda]}\left(\vw^t/\eta_t+\mA\T\valpha\right):=\left(\min\left(|y_j|,\lambda\right)\frac{y_j}{|y_j|}\right)_{j=1}^n,
\end{align}
where $(y_j)_{j=1}^n$ denotes an $n$-dimensional vector whose $j$th element is
given by $y_j$.
Note that the ratio $y_j/|y_j|$ is defined to be zero\footnote{This
is equivalent to defining $y_j/|y_j|={\rm sign}(y_j)$. We use
$y_j/|y_j|$ instead of ${\rm sign}(y_j)$ to define the
soft-threshold operations corresponding to $\ell_1$ and the
group-lasso regularizations (see \Secref{sec:dalgl}) in a similar way.} if $y_j=0$.
Substituting the above $\vv^t$ back into \Eqref{eq:dal_update_basic}, we obtain the following update equation:
\begin{align*}
 \vw^{t+1}&=\ST{\lambda\eta_t}(\vw^t+\eta_t\mA\T\valpha^t),
\end{align*}
where $\ST{\lambda\eta_t}$ is called the soft-threshold
operation\footnote{This notation is a simplified version of the general
notation we introduce later in \Eqref{eq:proximation}.} and is defined  as follows:
\begin{align}
\label{eq:softth}
\ST{\lambda}(\vy)&:=\left(\max(|y_j|-\lambda,0)\frac{y_j}{|y_j|}
\right)_{j=1}^n.
\end{align}
The soft-threshold operation is well known in signal processing
community and has been studied extensively~\citep{Don95,FigNow03,DauDefMol04,ComWaj05}.

Furthermore, substituting the above $\vv^t(\valpha)$ into
\Eqref{eq:ALfunc_basic}, we can express $\valpha^t$ as the minimizer of the
function
\begin{align}
\label{eq:dalL1inner}
\varphi_t(\valpha)&:=-L_{\eta_t}(\valpha,\vv^t(\valpha);\vw^t)=f_{\ell}^{\ast}(-\valpha)+\frac{1}{2\eta_t}\|\ST{\lambda\eta_t}(\vw^t+\eta_t\mA\T\valpha)\|^2,
\end{align}
which we also call an AL function with a slight abuse of terminology.
Note that the maximization in \Eqref{eq:ALfunc_basic} is turned into a
minimization of the above function by negating the AL function.

\section{Proximal minimization view}
\label{sec:proximal_view}
The first contribution of this paper is to derive DAL algorithm we
reviewed in \Secref{sec:dalreview} from the
proximal minimization framework \citep{Roc76a}, which allows for a new
interpretation of the algorithm (see \Secref{sec:dalprox}) and rigorous analysis of its convergence
(see \Secref{sec:analysis}). 

\subsection{Proximal minimization algorithm}
Let us consider the following iterative algorithm called the proximal
minimization algorithm~\citep{Roc76a} for the minimization
of the objective~\eqref{eq:problem}.
\begin{enumerate}
 \item Choose some initial solution $\vw^0$ and a sequence of
       non-decreasing positive numbers $\eta_0\leq\eta_1\leq\cdots$.
 \item Repeat until some criterion (e.g., duality
       gap~\citep{WriNowFig09,TomSug09}) is satisfied:
\begin{align}
\label{eq:ppa}
 \vw^{t+1}&=\argmin_{\vw\in\RR^n}\left(f(\vw)+\frac{1}{2\eta_{t}}\|\vw-\vw^t\|^2\right),
\end{align}
where $f(\vw)$ is the objective function in \Eqref{eq:problem} and
       $\eta_{t}$ controls the influence of the additional {\em proximity term}.
\end{enumerate}
The proximity term tries to keep the next solution $\vw^{t+1}$ close to
the current solution $\vw^{t}$. Importantly, the
objective~\eqref{eq:ppa} is strongly convex even if the original
objective~\eqref{eq:problem} is not; see \cite{Roc76b}.
Although at this point it is not clear how we are going to carry out the above
minimization, by definition we have
$f(\vw^{t+1})+\frac{1}{2\eta_t}\|\vw^{t+1}-\vw^{t}\|^2\leq f(\vw^{t})$;
i.e., provided that the step-size is positive, the function value
decreases monotonically at every iteration.

\subsection{Iterative shrinkage/thresholding algorithm from the proximal minimization framework}
\label{sec:ist}
The function to be minimized in \Eqref{eq:ppa} is strongly convex.
However, there seems to be
no obvious way to minimize \Eqref{eq:ppa}, because it is still
(possibly) non-differentiable and cannot be decomposed into smaller
problems because the elements of $\vw$ are coupled.

 One way to make the proximal
minimization algorithm practical is to linearly approximate (see
\cite{WriNowFig09}) the loss term at the current point $\vw^t$ as
\begin{align}
\label{eq:istapprox}
 f_{\ell}(\mA\vw) &\simeq f_{\ell}(\mA\vw^t)+\left(\nabla
 f_{\ell}^t\right)\T\mA\left(\vw-\vw^t\right),
\end{align}
where $\nabla f_{\ell}^t$ is a short hand for $\nabla f_{\ell}(\mA\vw^t)$.
Substituting the above approximation \eqref{eq:istapprox} into
the iteration \eqref{eq:ppa}, we obtain
\begin{align}
\label{eq:istupdatepre}
 \vw^{t+1}&=\argmin_{\vw\in\RR^n}\left(\left(\nabla
 f_{\ell}^t\right)\T\mA\vw+\phi_{\lambda}(\vw)+\frac{1}{2\eta_{t}}\|\vw-\vw^t\|^2\right),
\end{align}
where constant terms are omitted from the right-hand side.
Note that because of
the linear approximation, there is no coupling between the elements of
$\vw$. For example, if $\phi_{\lambda}(\vw)=\lambda\|\vw\|$, the
minimand in the right-hand side of the above equation can be separated
into $n$ terms each containing single $w_j$, 
which can be separately minimized.

Rewriting the above update equation, we obtain the well-known iterative
shrinkage/ thresholding (IST)
method\footnote{It is also known as the forward-backward splitting
method~\citep{LioMer79,ComWaj05,DucSin09}; see \Secref{sec:algorithms}.}~\citep{FigNow03,DauDefMol04,ComWaj05,FigBioNow07}. The IST
iteration can be written as follows: 
\begin{align}
\label{eq:istupdate}
  \vw^{t+1}&:=\prox{\phi_{\lambda\eta_t}}\left(\vw^t-\eta_t\mA\T\nabla f_{\ell}^t\right),
\end{align}
where the proximity operator ${\rm prox}_{\phi_{\lambda\eta_t}}$ is defined as follows:
\begin{align}
\label{eq:proximation}
\prox{\phi_\lambda}(\vy)=\argmin_{\vx\in\RR^n}\left(\frac{1}{2}\|\vy-\vx\|^2+\phi_{\lambda}(\vx)\right).
\end{align}
Note that the soft-threshold operation
$\prox{\lambda}^{\ell_1}$~\eqref{eq:softth} is the proximity operator
corresponding to the $\ell_1$-regularizer
$\phi_{\lambda}^{\ell_1}(\vw)=\lambda\|\vw\|_1$. 

\subsection{DAL algorithm from the proximal minimization framework}
\label{sec:dalprox}
The above IST approach can be considered to be constructing a linear lower bound
of the loss term  in \Eqref{eq:ppa} at the {\em current point}
$\vw^t$. In this subsection we show that we can precisely (to finite
precision) minimize \Eqref{eq:ppa} using a {\em parametrized}  linear lower
bound that can be adjusted to be the tightest at the {\em next point} $\vw^{t+1}$.
Our approach is based on the convexity of the loss function
$f_{\ell}$. First note that we can rewrite the loss function $f_{\ell}$
as a point-wise maximum as follows: 
\begin{align}
\label{eq:lbfell}
 f_{\ell}(\mA\vw)=\max_{\valpha\in\RR^m}\left((-\valpha)\T\mA\vw-f_{\ell}^\ast(-\valpha)\right),
\end{align}
where $f_{\ell}^{\ast}$ is the convex
conjugate functions of $f_{\ell}$.
Now we substitute this expression into the iteration \eqref{eq:ppa} as follows:
\begin{align}
\label{eq:minmaxdal}
\vw^{t+1}&= \argmin_{\vw\in\RR^n}\max_{\valpha\in\RR^m}\left\{-\valpha\T\mA\vw-f_{\ell}^\ast(-\valpha)+\phi_{\lambda}(\vw)+\frac{1}{2\eta_t}\|\vw-\vw^t\|^2\right\}.
\end{align}
Note that now the loss term is expressed as a {\em
 linear} function as in the IST
 approach; see \Eqref{eq:istupdatepre}. Now we exchange the order of minimization
 and maximization because the function to be minimaxed in
 \Eqref{eq:minmaxdal} is a saddle function (i.e., convex with respect to $\vw$ and
 concave with respect to $\valpha$~\citep{Roc70}), as follows:
\begin{align}
&\min_{\vw\in\RR^n}\max_{\valpha\in\RR^m}\left\{-\valpha\T\mA\vw-f_{\ell}^\ast(-\valpha)+\phi_{\lambda}(\vw)+\frac{1}{2\eta_t}\|\vw-\vw^t\|^2\right\}\nonumber\\
\label{eq:maxmindal}
&\qquad=\max_{\valpha\in\RR^m}\left\{-f_{\ell}^\ast(-\valpha)+\min_{\vw\in\RR^n}\left(-\valpha\T\mA{\vw}+\phi_{\lambda}(\vw)+\frac{1}{2\eta_t}\|\vw-\vw^t\|^2\right)\right\}.
\end{align}
Notice the similarity between the two minimizations
 \eqref{eq:istupdatepre} and
 \eqref{eq:maxmindal} (with fixed $\valpha$).

The minimization with respect to
$\vw$ in \Eqref{eq:maxmindal} gives the following update equation
\begin{align}
\label{eq:dalupdate}
\vw^{t+1}=\prox{\phi_{\lambda\eta_t}}\left(\vw^t+\eta_t\mA\T\valpha^t\right),
\end{align}
where $\valpha^t$ denotes the maximizer with respect to $\valpha$ in
\Eqref{eq:maxmindal}. Note that $\valpha^t$ is in general different from
$-\nabla f_{\ell}^t$ used in the IST approach~\eqref{eq:istupdate}. Actually, we show below
that $\valpha^t=-\nabla f_{\ell}^{t+1}$ if the max-min
problem~\eqref{eq:maxmindal} is solved exactly. Therefore taking
$\valpha^t=-\nabla f_{\ell}^{t}$ can be considered as a naive
approximation to this.

The final step to derive DAL algorithm is to compute the maximizer
$\valpha^t$ in \Eqref{eq:maxmindal}. This step is slightly involved and the derivation is
presented in Appendix~\ref{sec:deriv_dalinner}.
The result of the derivation can be written as follows
(notice that the maximization in \Eqref{eq:maxmindal} is turned into a
minimization by reversing the sign):
\begin{align}
\label{eq:dalinner}
\valpha^t =\argmin_{\valpha\in\RR^m}\Bigl(\underbrace{f_{\ell}^\ast(-\valpha)+\frac{1}{\eta_t}\Phi_{\lambda\eta_t}^{\ast}(\vw^{t}+\eta_t\mA\T\valpha)}_{=:\varphi_t(\valpha)}\Bigr),
\end{align}
where the function $\Phi_{\lambda\eta_t}^\ast$ is called the Moreau
envelope of $\phi_{\lambda}^\ast$ (see \cite{Mor65,Roc70}) and is defined as follows:
\begin{align}
 \label{eq:envelope}
\Phi_{\lambda}^\ast(\vw)&=\min_{\vx\in\RR^n}\left(\phi_{\lambda}^\ast(\vx)+\frac{1}{2}\|\vx-\vw\|^2\right).
\end{align}
We call the function $\varphi_t(\valpha)$ in \Eqref{eq:dalinner} the
augmented Lagrangian (AL) function. %

What we need to do at every iteration is to minimize the AL
function $\varphi_t(\valpha)$  and update the 
Lagrangian multiplier $\vw^t$ as in \Eqref{eq:dalupdate} using
the minimizer $\valpha^t$ in \Eqref{eq:dalinner}.
Of course in practice we would like to stop
the inner minimization at a finite tolerance. We discuss the stopping
condition in \Secref{sec:analysis}. 

The algorithm we derived above is indeed a generalization of DAL algorithm
we reviewed in \Secref{sec:dalreview}.  This can be shown by computing
the proximity operator~\eqref{eq:dalupdate} and the Moreau
envelope~\eqref{eq:envelope} for the specific case of
$\ell_1$-regularization; see \Secref{sec:dall1} and also Table~\ref{tab:regularizers}.

The AL function $\varphi_t(\valpha)$ is continuously differentiable, 
because the AL function is a sum of $f_{\ell}^\ast$ (differentiable by
assumption) and an envelope function (differentiable; see
Appendix~\ref{sec:basics}). In fact, using Lemma~\ref{lem:moreau_deriv} in
Appendix~\ref{sec:basics}, the derivative of the AL function can be  
evaluated as follows:
\begin{align}
\label{eq:ALderiv1}
\nabla\varphi_t(\valpha) &=-\nabla
 f_{\ell}^\ast(-\valpha)+\mA\vw^{t+1}(\valpha),
\end{align}
where
$\vw^{t+1}(\valpha):=\prox{\phi_{\lambda\eta_t}}\bigl(\vw^t+\eta_t\mA\T\valpha\bigr)$.
The expression for the second derivative
depends on the particular regularizer chosen.

 Notice again that the above update equation \eqref{eq:dalupdate} is very similar to the one in the IST approach (\Eqref{eq:istupdate}). However,
$-\valpha$, which is the slope of the lower-bound \eqref{eq:lbfell} is
optimized in the inner minimization~\eqref{eq:dalinner} so that the
lower-bound is the tightest at the {\em next point} $\vw^{t+1}$. 
In fact, if $\nabla \varphi_t(\valpha)=0$ then $\nabla f_{\ell}(\mA\vw^{t+1})= -\valpha^{t}$ because
of \Eqref{eq:ALderiv1} and $\nabla f_{\ell}(\nabla
f_{\ell}^\ast(-\valpha^{t}))=-\valpha^{t}$.
 The difference between the strategies used in IST and DAL to
construct a lower-bound is highlighted in \Figref{fig:ISTvsDAL}. IST
uses a fixed gradient-based lower-bound which is tightest at the current
solution $\vw^t$, whereas DAL uses a variational lower-bound, which can
be adjusted to become tightest at the next solution $\vw^{t+1}$.

\begin{figure}[tb]
 \begin{center}
  \subfigure[Gradient-based lower-bound used in IST]{\includegraphics[width=.4\textwidth]{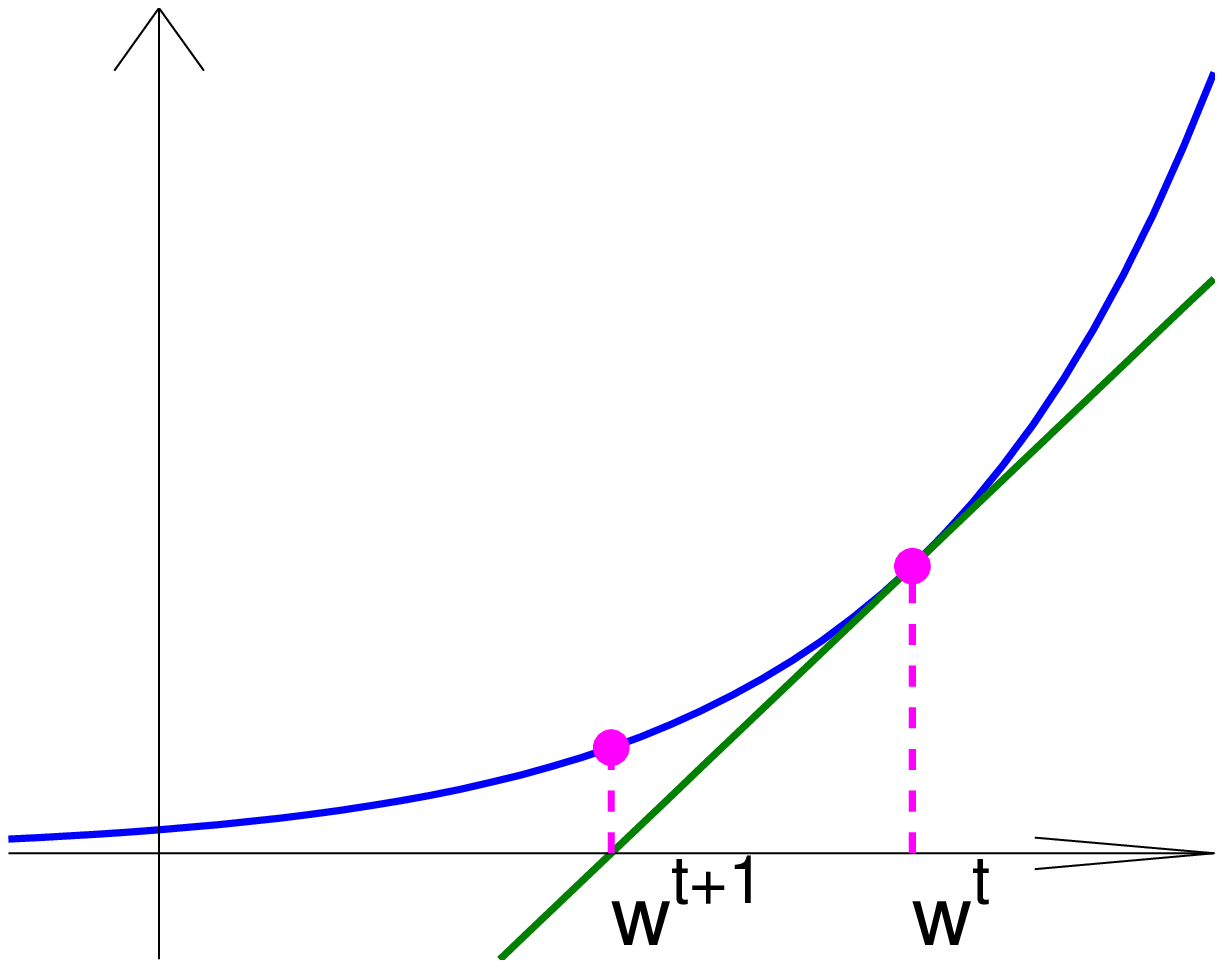}{\label{fig:gralb}}}
\hspace*{5mm}\subfigure[Variational lower-bound used in DAL]{\includegraphics[width=.4\textwidth]{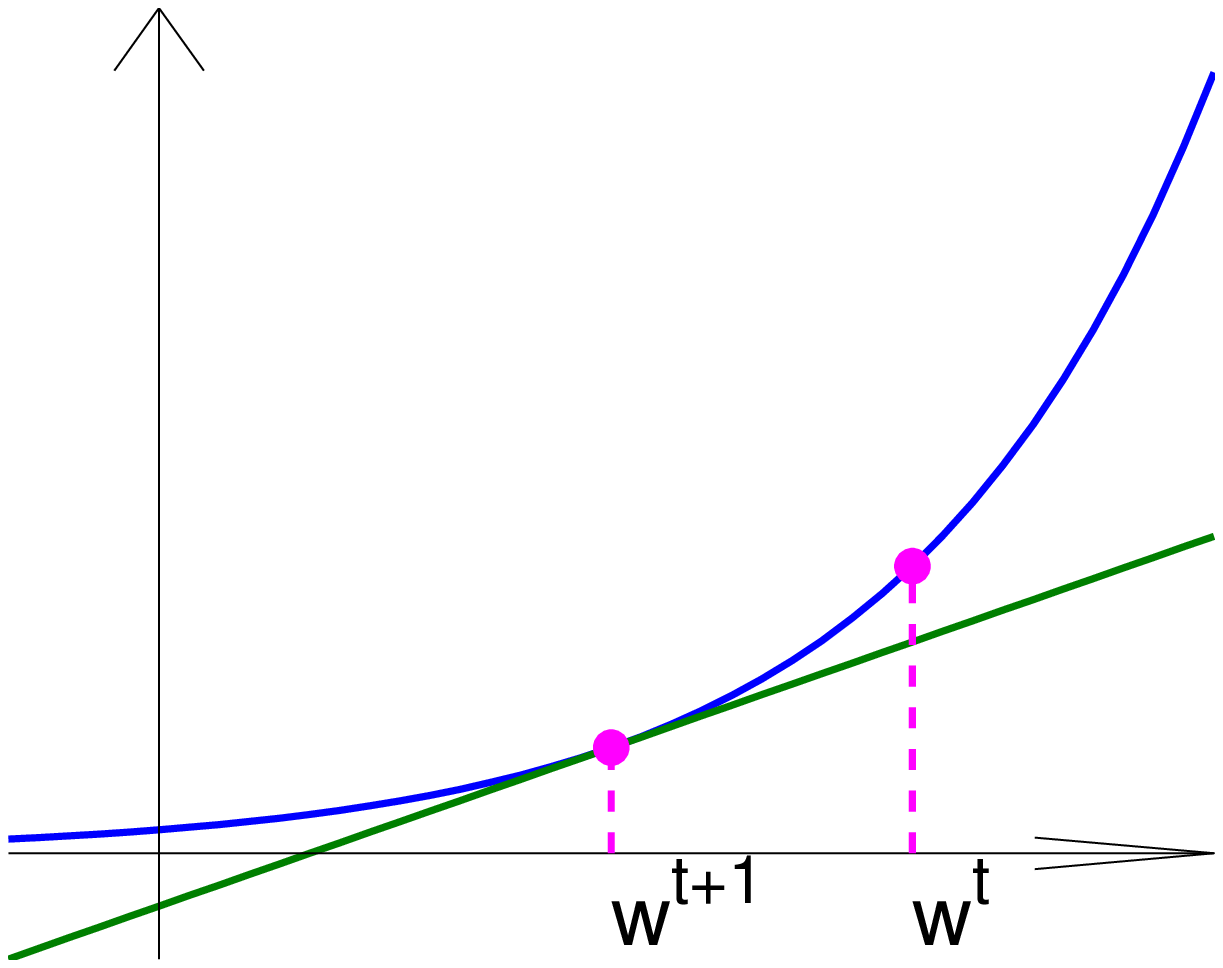}\label{fig:varlb}}
 \end{center}
 \caption{Comparison of the lower bounds used in IST and DAL.}
 \label{fig:ISTvsDAL}
\end{figure}
The general connection between the augmented Lagrangian algorithm and
the proximal minimization algorithm, and (asymptotic) convergence results can
be found in \cite{Roc76b,Ber82}. The derivation we show above is a special
case when the objective function $f(\vw)$ can be split into a part that is
easy to handle (regularization term $\phi_\lambda(\vw)$) and the rest
(loss term $f_{\ell}(\mA\vw)$).

\section{Exemplary instances}
\label{sec:instances}
In this section, we discuss special instances of DAL framework presented
in \Secref{sec:proximal_view} and qualitatively discuss the efficiency of minimizing
the inner objective. We first discuss the simple case of
$\ell_1$-regularization (\Secref{sec:dall1}), and then group-lasso
(\Secref{sec:dalgl}) and other more general regularization using the
so-called support functions (\Secref{sec:dalsupport}). In addition, the
case of component-wise regularization is discussed in
\Secref{sec:indivreg}. See also Table~\ref{tab:regularizers} for a list
of regularizers.

\subsection{Dual augmented Lagrangian algorithm for $\ell_1$-regularization}
\label{sec:dall1}
For the $\ell_1$-regularization, $\phi_{\lambda}^{\ell_1}(\vw)=\lambda\|\vw\|_1$,
the update equation \eqref{eq:dalupdate} can be rewritten as follows:
\begin{align}
\label{eq:dalL1update}
 \vw^{t+1}&=\ST{\lambda\eta_t}\left(\vw^t+\eta_t\mA\T\valpha^t\right),
\end{align}
where $\ST{\lambda}$ is the proximity operator corresponding to
 the $\ell_1$-regularizer defined in \Eqref{eq:softth}.  Moreover,
 noticing that the convex conjugate of the $\ell_1$-regularizer is the
 indicator function $\delta_\lambda^\infty$ in \Eqref{eq:indicatorL1},
we can derive the envelope function $\Phi_{\lambda}^{\ast}$ in
 \Eqref{eq:envelope} as follows (see also \Figref{fig:envelopes}):
\begin{align*}
 \Phi_{\lambda}^\ast(\vw)=\frac{1}{2}\left\|\ST{\lambda}(\vw)\right\|^2.
\end{align*}
Therefore, the AL function \eqref{eq:dalL1inner}  in 
 \cite{TomSug09}  is derived from the proximal minimization
 framework (see \Eqref{eq:dalinner}) in \Secref{sec:proximal_view}.

\begin{algorithm}[tb]
\caption{DAL algorithm for $\ell_1$-regularization} 
\label{alg:DALL1}
\begin{algorithmic}[1]
 \STATE {\bf Input:} design matrix $\mA$, loss function $f_{\ell}$,
 regularization constant  $\lambda$, sequence of proximity parameters
 $\eta_t$ ($t=0,1,2,\ldots$), initial solution $\vw^0$, tolerance $\epsilon$.
 \STATE Set $t=0$.
 \REPEAT
 \STATE Minimize the augmented
 Lagrangian function $\varphi_t(\valpha)$ (see \Eqref{eq:dalL1inner})
 with the gradient and Hessian given in \Eqsref{eq:ALderiv1_l1} and
 \eqref{eq:ALderiv2_l1}, respectively, using Newton's
 method. Let $\valpha^t$ be the approximate minimizer
$$
\valpha^t\simeq \argmin_{\valpha\in\RR^m}\Bigl(
f_{\ell}^\ast(-\valpha)+\frac{1}{2\eta_t}\left\|\ST{\lambda\eta_t}(\vw^t+\eta_t\mA\T\valpha)\right\|^2
\Bigr),
$$
with the stopping criterion (see \Secref{sec:approx})
$$
\|\nabla\varphi_t(\valpha^t)\|\leq\sqrt{\frac{\gamma}{\eta_t}}\left\|\ST{\lambda\eta_t}(\vw^t+\eta_t\mA\T\valpha^t)-\vw^{t}\right\|,
$$
where $\varphi_t(\valpha)$ is the derivative of the inner objective \eqref{eq:ALderiv1_l1}.
\STATE Update $\vw^{t+1}:=\ST{\lambda\eta_t}(\vw^t+\eta_t\mA\T\valpha^t)$, $t\leftarrow t+1$.
 \UNTIL{relative duality gap (see \Secref{sec:gap}) is less than the tolerance $\epsilon$.}
\STATE {\bf Output:} the final solution $\vw^{t}$.
\end{algorithmic}
\end{algorithm}


We use Newton's method for the minimization of the inner objective
$\varphi_t(\valpha)$. The overall algorithm is shown in
Algorithm~\ref{alg:DALL1}.  The gradient and Hessian of
the AL function \eqref{eq:dalL1inner} can be
evaluated as follows~\citep{TomSug09}:  
\begin{align}
\label{eq:ALderiv1_l1}
\nabla\varphi_t(\valpha) &=-\nabla
 f_{\ell}^\ast(-\valpha)+\mA\vw^{t+1}(\valpha),\\
\label{eq:ALderiv2_l1}
\nabla^2\varphi_t(\valpha) &=\nabla^2 f_{\ell}^\ast(-\valpha)+\eta_t\mA_+\mA_+\T,
\end{align}
where
$\vw^{t+1}(\valpha):=\ST{\lambda\eta_t}(\vw^t+\eta_t\mA\T\valpha)$, and $\mA_+$ is the matrix that consists of columns of $\mA$ that
corresponds to ``active'' variables (i.e., the non-zero elements of
$\vw^{t+1}(\valpha)$ ). Note that \Eqref{eq:ALderiv1_l1} equals the
general expression \eqref{eq:ALderiv1} from the proximal minimization framework.

It is worth noting that in both the computation of
matrix-vector product in \Eqref{eq:ALderiv1_l1} and the computation of
matrix-matrix product in \Eqref{eq:ALderiv2_l1}, the cost is only
proportional to the number of non-zero elements of $\vw^{t+1}(\valpha)$.
Thus when we are aiming for a sparse solution, the minimization of
\Eqref{eq:dalL1inner} can be performed efficiently.

%


\subsection{Group lasso}
\label{sec:dalgl}
Let $\phi_{\lambda}$ be the group-lasso penalty~\citep{YuaLin06}, i.e.,
\begin{align}
\label{eq:grouplasso}
 \phi_{\lambda}^{\Gf}(\vw)=\lambda\sum_{\gf\in\Gf}\|\vw_{\gf}\|,
\end{align}
where $\Gf$ is a disjoint partition of
the index set $\{1,\ldots,n\}$, and $\vw_{\gf}\in\RR^{|\gf|}$ is a sub-vector of $\vw$
that consists of rows of $\vw$ indicated by $\gf\subseteq\{1,\ldots,n\}$. 
The proximity operator corresponding to the group-lasso
regularizer $\phi_{\lambda}^{\Gf}$ is
obtained as follows:
\begin{align}
 \label{eq:prox_gl}
\prox{\lambda}^{\Gf}(\vy):=\prox{\phi_{\lambda}^{\Gf}}(\vy)=\left(\max(\|\vy_{\gf}\|-\lambda,0)\frac{\vy_{\gf}}{\|\vy_{\gf}\|}\right)_{\gf\in\Gf},
\end{align}
where similarly to \Eqref{eq:softth}, $(\vy_{\gf})_{\gf\in\Gf}$ denotes
an $n$-dimensional vector whose $\gf$ component is given by
$\vy_{\gf}$. 
Moreover, analogous to update equation \eqref{eq:dalL1update} (see also
\Eqref{eq:dalL1inner}) in the $\ell_1$-case, the update equations can be
written as follows: 
\begin{align}
\vw^{t+1}&=\prox{\lambda\eta_t}^{\Gf}\left(\vw^t+\eta_t\mA\T\valpha^t\right),\nonumber
\intertext{where $\valpha^t$ is the minimizer of the AL function}
\label{eq:ALfunc_gl}
\varphi_t(\valpha)&=f_{\ell}^\ast(-\valpha)+\frac{1}{2\eta_t}\|\prox{\lambda\eta_t}^{\Gf}(\vw^t+\eta_t\mA\T\valpha)\|^2.
\end{align}
The overall algorithm is obtained by replacing the soft-thresholding operations
in Algorithm~\ref{alg:DALL1} by the one defined above~\eqref{eq:prox_gl}.
 In addition, the gradient and Hessian of the AL function
 $\varphi_t(\valpha)$ can be  written as follows:
\begin{align}
\label{eq:ALderiv1_gl}
\nabla\varphi_t(\valpha)&=-\nabla f_{\ell}^\ast(-\valpha)+\mA\vw^{t+1}(\valpha),\\
\label{eq:ALderiv2_gl}
\nabla^2\varphi_t(\valpha)&=\nabla^2 f_{\ell}^\ast(-\valpha)+\eta_t\sum_{\gf\in\Gf^+}\mA_{\gf}\left(\left(1-\frac{\lambda\eta_t}{\|\vq_{\gf}\|}\right)\mI_{|\gf|}+\frac{\lambda\eta_t}{\|\vq_{\gf}\|}\tilde{\vq}_{\gf}\tilde{\vq}_{\gf}\T\right)\mA_{\gf}\T,
\end{align}
where
$\vw^{t+1}(\valpha)=\prox{\lambda\eta_t}^{\Gf}(\vw^{t}+\eta_t\mA\T\valpha)$, and
$\Gf^+$ is a subset of $\Gf$ that consists of active groups, namely
$\Gf^+:=\{\gf\in\Gf: \|\vw_{\gf}^{t+1}(\valpha)\|>0\}$; 
 $\mA_{\gf}$ is a sub-matrix of $\mA$ that consists of columns of
$\mA$ that corresponds to the index-set $\gf$;
$\mI_{|\gf|}$ is the $|\gf|\times|\gf|$ identity matrix; the vector $\vq\in\RR^n$ is
defined as $\vq:=\vw^{t}+\eta_t\mA\T\valpha$ and
$\tilde{\vq}_{\gf}:=\vq_{\gf}/\|\vq_{\gf}\|$, where $\vq_{\gf}$ is
defined analogously to $\vw_{\gf}$. Note that in the above expression,
$\lambda\eta_t/\|\vq_{\gf}\|\leq 1$ for $\gf\in\Gf^+$ by the
soft-threshold operation \eqref{eq:prox_gl}.

Similarly to the $\ell_1$-case in the last
subsection, the sparsity of $\vw^{t+1}(\valpha)$ (i.e., $|\Gf^+|\ll |\Gf|$)
can be exploited to efficiently compute the gradient
\eqref{eq:ALderiv1_gl} and the Hessian \eqref{eq:ALderiv2_gl}.

\subsection{Support functions}
\label{sec:dalsupport}
The $\ell_1$-norm regularization and the group lasso regularization
in \Eqref{eq:grouplasso} can be generalized to the class of support
functions. The support function of a convex set $C_{\lambda}$ is defined
as follows:
\begin{align}
\label{eq:support}
\phi_\lambda(\vx)=\sup_{\vy\in C_{\lambda}}\vx\T\vy.
\end{align}
For example, the $\ell_1$-norm is the support function of the
$\ell_\infty$ unit ball (see \cite{Roc70}) and the group lasso
regularizer \eqref{eq:grouplasso} is the support function of the
group-generalized $\ell_\infty$-ball defined as
$\{\vy\in\RR^n:\|\vy_{\gf}\|\leq \lambda,\,\forall\gf\in\Gf\}$.
It is well known that the convex conjugate of the support function \eqref{eq:support} is the indicator function of $C$ (see \cite{Roc70}), namely,
\begin{align}
\label{eq:indicator}
\phi_{\lambda}^\ast(\vy)&=
\begin{cases}
 0 & (\textrm{if $\vy\in C_{\lambda}$}),\\
 +\infty & (\textrm{otherwise}).
\end{cases}
\end{align}
The proximity operator corresponding to the support function
\eqref{eq:support} can be written as follows:
\begin{align*}
 \prox{C_\lambda}^{\rm sup}(\vy):=\vy - {\rm proj}_{C_{\lambda}}(\vy),
\end{align*}
where ${\rm proj}_{C_{\lambda}}$ is the projection onto $C_{\lambda}$;
see Lemma~\ref{lem:moreau_decomp}  in Appendix~\ref{sec:basics}.
Finally, by computing the Moreau envelope
 \eqref{eq:envelope} corresponding to the above $\phi_{\lambda}^{\ast}$,  we have
\begin{align}
\label{eq:ALfunc_sp}
\varphi_t(\valpha)&=f_{\ell}^\ast(-\valpha)+\frac{1}{2\eta_t}\|\prox{C_{\lambda\eta_t}}^{\rm
 sup}(\vw^t+\eta_t\mA\T\valpha)\|^2,
\end{align}
where we used the fact that for the indicator function in \Eqref{eq:indicator},
$\phi_{\lambda}^\ast({\rm proj}_{C_{\lambda}}(\vz))=0$ ($\forall \vz$) and
Lemma~\ref{lem:moreau_decomp}.
Note that letting $\prox{C_{\lambda}}^{\rm sup}=\ST{\lambda}$ and
$\prox{C_{\lambda}}^{\rm sup}=\prox{\lambda}^{\Gf}$ in 
\Eqref{eq:ALfunc_sp}, we obtain \Eqref{eq:dalL1inner} and
\Eqref{eq:ALfunc_gl}, respectively.



\subsection{Handling different regularization constant for each component}
\label{sec:indivreg}
The $\ell_1$-regularizer in \Secref{sec:dall1} and the group lasso
regularizer in \Secref{sec:dalgl} assume that all the components
(variables or groups) are regularized by the same constant
$\lambda$. However the general formulation in \Secref{sec:dalprox} allows
 using different regularization constant for each component.

 For example, let us consider the following regularizer:
\begin{align}
\label{eq:indivL1}
\phi_{\bm{\lambda}}(\vw)= \sum_{j=1}^n\lambda_j|w_j|,
\end{align}
where $\lambda_j\geq 0$ ($j=1,\ldots,n$). Note that we can also include
unregularized terms (e.g., a bias term) by setting the corresponding
regularization constant $\lambda_j=0$.
The soft-thresholding operation corresponding to the regularizer~\eqref{eq:indivL1} is
written as follows:
\begin{align*}
\prox{\bm{\lambda}}^{\ell_1}(\vy)
 &=\left(\max(|y_j|-\lambda_j,0)\frac{y_j}{|y_j|}\right)_{j=1}^n,
\end{align*}
where again the ratio $y_j/|y_j|$ is defined to be zero if $y_j=0$.
Note that if $\lambda_j=0$, the soft-thresholding operation is an
identity mapping for that component. Moreover, by noticing that the 
regularizer \eqref{eq:indivL1} is a support function (see \Secref{sec:dalsupport}), the envelope 
function $\Phi_{\lambda}^\ast$ in \Eqref{eq:envelope} is written as
follows:
\begin{align*}
 \Phi_{\bm{\lambda}}^{\ast}(\vw)&=\frac{1}{2}\sum_{j=1}^n{\max}^2(|w_j|-\lambda_j,0),
\end{align*}
which can also be derived by noticing that
$\Phi_{\bm{\lambda}}^\ast(0)=0$ and $\nabla
\Phi_{\bm{\lambda}}^\ast(\vy)=\prox{\bm{\lambda}}^{\ell_1}(\vy)$
(Lemma~\ref{lem:moreau_deriv} in Appendix~\ref{sec:basics}).

As a concrete example, let $b$ be an unregularized bias term and let us
assume that all the components of $\vw\in\RR^n$ are regularized by the
same regularization constant $\lambda$. In other words, we aim to solve
the following optimization problem:
\begin{align*}
 \minimize_{\vw\in\RR^n,b\in\RR}&\quad f_{\ell}(\mA\vw+\bm{1}_mb)+\lambda\|\vw\|_1,
\end{align*}
where $\|\vw\|_1$ is the $\ell_1$-norm of $\vw$, and $\bm{1}_m$ is an
   $m$-dimensional all one vector.
   The update equations
\eqref{eq:dalupdate} and \eqref{eq:dalinner} can be written as follows:
\begin{align}
\label{eq:dalL1update_w_bias}
 \vw^{t+1}&=\ST{\lambda\eta_t}(\vw^t+\eta_t\mA\T\valpha^t),\\
\label{eq:dalL1_update_b_bias}
 b^{t+1}&=b^t+\eta_t\bm{1}_m\T\valpha^t,
\end{align}
where $\valpha^t$ is the minimizer of the AL function as follows:
\begin{align}
\label{eq:dalL1inner_bias}
 \valpha^t&=\argmin_{\valpha\in\RR^m}\left(
f_{\ell}^\ast(-\valpha)+\frac{1}{2\eta_t}\left(\|\ST{\lambda\eta_t}(\vw^t+\eta_t\mA\T\valpha)\|^2+(b^t+\eta_t\bm{1}_m\T\valpha)^2\right)
\right).
\end{align}

\section{Analysis}
\label{sec:analysis}
In this section, we first show the convergence of DAL algorithm 
assuming that the inner minimization problem \eqref{eq:dalinner} is solved
exactly (\Secref{sec:exact}), which is equivalent to the proximal minimization algorithm \eqref{eq:ppa}. The convergence is presented both in terms of 
the function value and the norm of the residual.  Next, since it is
impractical to perform the inner minimization to high precision, the
finite tolerance version of the two theorems are presented in
\Secref{sec:approx}. The convergence rate obtained in
\Secref{sec:approx} is slightly worse than the exact case. In
\Secref{sec:fasterrate}, we show that the convergence rate can be
improved by performing the inner minimization more precisely.
 Most of the proofs are given in Appendix~\ref{sec:proofs} for the sake of readability.

 Our result is inspired partly by \cite{BecTeb09} and is similar to the
 one given in \cite{Roc76a} and \cite{KorBer76}. However,  our analysis
 does not require asymptotic arguments as in \cite{Roc76a} or rely on 
the strong convexity of the objective as in \cite{KorBer76}.
Importantly the stopping criterion we discuss in \Secref{sec:approx} can
be checked in practice. Key to our
analysis is the Lipschitz  continuity of the gradient of the loss
function~$\nabla  f_{\ell}$ and the assumption that the proximation with
respect to $\phi_{\lambda}$ (see \Eqref{eq:proximation}) can be computed exactly.
Connections between our assumption and the ones made in earlier studies
are discussed in \Secref{sec:validity_assumption}.

\subsection{Exact inner minimization}
\label{sec:exact}

\begin{lemma}[\cite{BecTeb09}]
\label{lem:key_exact}
Let $\vw^1,\vw^2,\ldots$ be the sequence generated by the proximal
 minimization algorithm (\Eqref{eq:ppa}). For arbitrary $\vw\in\RR^n$ we have
\begin{align}
\label{eq:key_exact}
\eta_t(f(\vw^{t+1})-f(\vw))\leq \frac{1}{2}\|\vw^{t}-\vw\|^2-\frac{1}{2}\|\vw^{t+1}-\vw\|^2.
\end{align}
\end{lemma}
\begin{proof}
First notice that $(\vw^t-\vw^{t+1})/\eta_t\in\partial
 f(\vw^{t+1})$ because $\vw^{t+1}$ minimizes \Eqref{eq:ppa}. Therefore
 using the convexity of $f$, we have\footnote{We use the notation
 $\dot{\vx}{\vy}:=\sum_{j=1}^nx_jy_j$ for $\vx,\vy\in\RR^n$.}
\begin{align}
\label{eq:proof1_convexity}
 \eta_t(f(\vw)-f(\vw^{t+1}))&\geq\dot{\vw-\vw^{t+1}}{\vw^{t}-\vw^{t+1}}\\
&=\dot{\vw-\vw^{t+1}}{\vw^{t}-\vw+\vw-\vw^{t+1}}\nonumber\\
&\geq\|\vw-\vw^{t+1}\|^2-\|\vw-\vw^{t+1}\|\|\vw^t-\vw\|\nonumber\\
&\geq\frac{1}{2}\|\vw-\vw^{t+1}\|^2-\frac{1}{2}\|\vw^t-\vw\|^2,\nonumber
\end{align}
where the third line follows from Cauchy-Schwartz inequality and the
 last line follows from the inequality of arithmetic and 
 geometric means.
\end{proof}

Note that DAL algorithm (\Eqsref{eq:dalupdate} and
\eqref{eq:dalinner}) with exact inner minimization generates a sequence
from the proximal minimization algorithm (\Eqref{eq:ppa}). Therefore we
have the following theorem.

\begin{theorem}
\label{thm:convergence_f}
Let $\vw^1,\vw^2,\ldots$ be the sequence generated by DAL algorithm
(\Eqsref{eq:dalupdate} and \eqref{eq:dalinner}); let $W^\ast$ be the
 set of minimizers of the objective \eqref{eq:problem} and let
 $f(W^{\ast})$ denote the minimum objective value.  If the  inner minimization
(\Eqref{eq:dalinner}) is solved exactly and the proximity parameter $\eta_t$ is
 increased exponentially, then DAL algorithm converges exponentially as follows:
\begin{align}
\label{eq:convergence_f}
 f(\vw^{k+1})-f(W^\ast)\leq \frac{\|\vw^0-W^{\ast}\|^2}{2C_k},
\end{align}
where $\|\vw^0-W^{\ast}\|$ denotes the minimum distance between the
 initial solution $\vw^0$ and $W^{\ast}$, namely,
 $\|\vw^0-W^{\ast}\|=\min_{\vw^{\ast}\in W^{\ast}}\|\vw^0-\vw^{\ast}\|$.
 Note that $C_k=\sum_{t=0}^k\eta_t$ also grows exponentially.
\end{theorem}
\begin{proof}
Let $\vw^{\ast}$ be any point in $W^{\ast}$. 
 Substituting $\vw=\vw^{\ast}$ in \Eqref{eq:key_exact} and summing both sides
from $t=1$ to $t=k$, we have
\begin{align*}
 \left(\textstyle\sum_{t=0}^k\eta_t\right)\left(\sum_{t=0}^k\frac{\eta_tf(\vw^{t+1})}{\sum_{t=0}^k\eta_t}-f(\vw^\ast)\right)&\leq\frac{1}{2}\|\vw^0-\vw^\ast\|^2-\frac{1}{2}\|\vw^{k+1}-\vw^\ast\|^2\\
&\leq\frac{1}{2}\|\vw^0-\vw^{\ast}\|^2.
\intertext{In addition, since $f(\vw^{t+1})\leq f(\vw^t)$
 ($t=0,1,2,\ldots$) from \Eqref{eq:ppa},  we have}
 \left(\textstyle\sum_{t=0}^k\eta_t\right)\left(f(\vw^{k+1})-f(\vw^{\ast})\right)&\leq\frac{1}{2}\|\vw^0-\vw^{\ast}\|^2.
\end{align*} 
Finally, taking the minimum of the right-hand side with respect to
 $\vw^{\ast}\in W^{\ast}$ and using the equivalence of proximal
 minimization \eqref{eq:ppa} and DAL algorithm
 \eqref{eq:dalupdate}-\eqref{eq:dalinner} (see \Secref{sec:dalprox}),
 we complete the proof.  
\end{proof}

The above theorem claims the convergence of the residual function values
$f(\vw^t)-f(\vw^\ast)$ obtained along the sequence
$\vx_1,\vx_2,\ldots$. We can convert the above result into convergence
in terms of the residual norm $\|\vw^t-\vw^\ast\|$ by 
introducing an assumption that connects the residual function value
to the residual norm. In addition, we slightly generalize
Lemma~\ref{lem:key_exact} to improve the convergence
rate. Consequently, we obtain the following theorem.

\begin{theorem}
\label{thm:convergence_w}
Let $\vw^1,\vw^2,\ldots$ be the sequence generated by DAL algorithm
(\Eqsref{eq:dalupdate} and \eqref{eq:dalinner}) and let $W^\ast$ be
the set of minimizers of the objective~\eqref{eq:problem}.  Let us assume
that there are a positive constant $\sigma$ and  a scalar $\alpha$
($1\leq \alpha\leq 2$) such that
\begin{flalign}
\label{eq:lowerbounded}
\textrm{\bf(A1)} && f(\vw^{t+1})-f(W^\ast) &\geq
 \sigma \|\vw^{t+1}-W^{\ast}\|^{\alpha}\qquad(t=0,1,2,\ldots), &&
\end{flalign}
where $f(W^{\ast})$ denotes the minimum objective value, and $\|\vw-W^{\ast}\|$ denotes the minimum distance between $\vw\in\RR^n$ and
the set of minimizers $W^{\ast}$ as $\|\vw-W^{\ast}\|:=\min_{\vw^\ast\in
 W^{\ast}}\|\vw-\vw^{\ast}\|$.

If the inner minimization is solved exactly, we have the following
inequality:
\begin{align*}
\|\vw^{t+1}-W^{\ast}\| +\sigma\eta_{t}\|\vw^{t+1}-W^{\ast}\|^{\alpha-1}\leq \|\vw^{t}-W^{\ast}\|.
\end{align*}
Moreover, this implies that
\begin{align}
\label{eq:thm1}
\|\vw^{t+1}-W^{\ast}\|^{\frac{1+(\alpha-1)\sigma\eta_{t}}{1+\sigma\eta_{t}}}\leq \frac{1}{1+\sigma\eta_{t}}\|\vw^{t}-W^{\ast}\|.
\end{align}
That is, $\vw^{t}$ converges to $W^\ast$ {\em super-linearly} if
 $\alpha<2$ or $\alpha=2$ and $\eta_{t}$ is increasing, in a {\em global
 and non-asymptotic} sense.
\end{theorem}
\begin{proof}
 See Appendix~\ref{sec:proof_convergence_w}.
\end{proof}
Note that the above super-linear convergence holds without the
assumption in Theorem~\ref{thm:convergence_f} that $\eta_t$ is
increased exponentially. 

\subsection{Approximate inner minimization}
\label{sec:approx}
First we present a finite tolerance version of Lemma~\ref{lem:key_exact}.
\begin{lemma}
\label{lem:key_approx}
Let $\vw^1,\vw^2,\ldots$ be the sequence generated by DAL algorithm
 (\Eqsref{eq:dalupdate} and \eqref{eq:dalinner}). Let us assume the
 following conditions.
\begin{description}
 \item[(A2)]  The loss function $f_{\ell}$ has a Lipschitz continuous gradient
 with modulus $1/\gamma$, i.e.,
\begin{align}
\label{eq:lipshitz}
 \left\|\nabla f_{\ell}(\vz)-\nabla f_{\ell}(\vz')\right\|\leq
\frac{1}{\gamma}\|\vz-\vz'\|\qquad(\forall \vz,\vz'\in\RR^m), 
\end{align}
\item[(A3)] The proximation with respect to $\phi_{\lambda}$ (see
	   \Eqref{eq:proximation}) can be  computed exactly.
\item[(A4)] The  inner minimization (\Eqref{eq:dalinner}) is solved to the
 following tolerance:
\begin{align}
 \label{eq:dalstopcond}
\|\nabla\varphi_t(\valpha^t)\|&\leq
 \sqrt{\frac{\gamma}{\eta_t}}\|\vw^{t+1}-\vw^{t}\|,
\end{align}
where $\gamma$ is the constant in \Eqref{eq:lipshitz}.
\end{description}
Under assumptions {\bf (A2)}--{\bf (A4)}, for  arbitrary $\vw\in\RR^n$ we have 
\begin{align}
\label{eq:key_approx}
\eta_t(f(\vw^{t+1})-f(\vw))\leq \frac{1}{2}\|\vw^{t}-\vw\|^2-\frac{1}{2}\|\vw^{t+1}-\vw\|^2.
\end{align}
\end{lemma}
\begin{proof}
 See Appendix~\ref{sec:proof_key_approx}.
\end{proof}
Note that Lemma~\ref{lem:key_approx} states that even with the weaker
stopping criterion {\bf (A4)}, we can obtain inequality
\eqref{eq:key_approx} as in Lemma~\ref{lem:key_exact}.

The assumptions we make here are rather weak. In Assumption {\bf (A2)}, the
loss function $f_{\ell}$ does not include the design matrix
$\mA$ (see Table~\ref{tab:loss}). Therefore, it is easy to compute the constant $\gamma$. 
Accordingly, the stopping criterion {\bf (A4)} can be checked without assuming
anything about the data.

Furthermore, summing
both sides of inequality~\eqref{eq:key_approx} and assuming that $\eta_t$ is
increased exponentially, we obtain
Theorem~\ref{thm:convergence_f} also under the approximate
minimization~{\bf (A4)}.

Finally, an analogue of Theorem~\ref{thm:convergence_w}, which does not
assume the exponential increase in $\eta_t$, is obtained as follows.

\begin{theorem}
\label{thm:convergence_w_approx}
 Let $\vw^1,\vw^2,\ldots$ be the sequence generated by DAL algorithm and let $W^{\ast}$
 be the set of  minimizers of the objective~\eqref{eq:problem}. Under assumption
 {\bf (A1)} in Theorem~\ref{thm:convergence_w} and 
{\bf (A2)}-{\bf (A4)} in Lemma~\ref{lem:key_approx}, we have
\begin{align}
\label{eq:convergence_w_approx_1}
 \|\vw^{t+1}-W^{\ast}\|^2+2\sigma\eta_t\|\vw^{t+1}-W^{\ast}\|^{\alpha}&\leq \|\vw^t-W^{\ast}\|^2,
\end{align}
where $\|\vw^{t}-W^{\ast}\|$ is the minimum distance between $\vw^{t}$ and
 $W^{\ast}$ as in Theorem~\ref{thm:convergence_w}.
Moreover, this implies that
\begin{align}
\label{eq:thm2}
 \|\vw^{t+1}-W^{\ast}\|^{\frac{1+\alpha\sigma\eta_t}{1+2\sigma\eta_t}}&\leq\frac{1}{\sqrt{1+2\sigma\eta_t}}\|\vw^{t}-W^{\ast}\|.
\end{align}
That is,  $\vw^t$ converges to $W^\ast$ super-linearly if $\alpha<2$ or
 $\alpha=2$ and $\eta_t$ is increasing.
\end{theorem}
\begin{proof}
Let $\bar{\vw}^t$ be the closest point in $W^{\ast}$ from $\vw^{t}$,
 i.e., $\bar{\vw}^t:=\argmin_{\vw^{\ast}\in
 W^{\ast}}\|\vw^{t}-\vw^{\ast}\|$. Using  Lemma~\ref{lem:key_approx}
 with $\vw=\bar{\vw}^t$  and Assumption ${\bf (A1)}$, we have the
 first part of the theorem as follows:
\begin{align*}
\|\vw^{t}-W^{\ast}\|^2=\|\vw^{t}-\bar{\vw}^t\|^2 &\geq
 \|\vw^{t+1}-\bar{\vw}^t\|^2+2\sigma\eta_t\|\vw^{t+1}-W^{\ast}\|^\alpha\\
&\geq \|\vw^{t+1}-W^{\ast}\|^2+2\sigma\eta_t \|\vw^{t+1}-W^{\ast}\|^{\alpha},
\end{align*}
where we used the minimality of $\|\vw^{t+1}-\bar{\vw}^{t+1}\|$ in the
second line.
The last part of the theorem~\eqref{eq:thm2} can be obtained in a similar manner as that of
 Theorem~\ref{thm:convergence_w} using Young's inequality (see Appendix~\ref{sec:proof_convergence_w}).
\end{proof}

\subsection{A faster rate}
\label{sec:fasterrate}
The factor $1/\sqrt{1+2\sigma\eta_t}$ obtained under the approximate
minimization~{\bf (A4)} (see inequality~\eqref{eq:thm2} in Theorem~\ref{thm:convergence_w_approx}) is larger than that
obtained under the exact inner minimization (see inequality~\eqref{eq:thm1} in Theorem~\ref{thm:convergence_w}); i.e., the
statement in 
Theorem~\ref{thm:convergence_w_approx} is weaker than that in Theorem~\ref{thm:convergence_w}.

Here we show that a better rate can also be obtained for approximate
minimization if we perform the inner minimization to
$O(\|\vw^{t+1}-\vw^{t}\|/\eta_t)$ instead of
$O(\|\vw^{t+1}-\vw^{t}\|/\sqrt{\eta_t})$ in Assumption~{\bf (A4)}.

\begin{theorem}
\label{thm:fasterrate}
 Let $\vw^1,\vw^2,\ldots$ be the sequence generated by DAL algorithm and
 let $W^\ast$
 be the set of minimizers of the objective~\eqref{eq:problem}. 
Under assumption {\bf (A1)} in Theorem~\ref{thm:convergence_w} with $\alpha=2$,
 and assumptions {\bf  (A2)} and {\bf (A3)} in
 Lemma~\ref{lem:key_approx}, for any $\epsilon<1$ such that
 $\delta:=(1-\epsilon)/(\sigma\eta_t)\leq 3/4$, if we solve the inner
 minimization to the following precision
\begin{flalign*}
\textrm{\bf (A4$'$)}&& \|\nabla\varphi_t(\valpha^t)\|&\leq
 \frac{\sqrt{\gamma(1-\epsilon)/\sigma}}{\eta_t}\|\vw^{t+1}-\vw^{t}\|, &&
\end{flalign*}
then we have
\begin{align*}
 \|\vw^{t+1}-W^{\ast}\|\leq\frac{1}{1+\epsilon\sigma\eta_t}\|\vw^{t}-W^{\ast}\|.
\end{align*}
\end{theorem}
\begin{proof}
 See Appendix~\ref{sec:proof_fasterrate}
\end{proof}
Note that the assumption $\delta<3/4$ is rather weak, because if the factor
$\delta$ is greater than one, the stopping criterion ${\bf (A4')}$ would
be weaker than the earlier criterion ${\bf (A4)}$. In order to be on the
safe side, we can choose $\epsilon=\max(\epsilon_0,1-3\sigma\eta_t/4)$
(assuming that we know the constant $\sigma$) and the above statement
holds with $\epsilon=\epsilon_0$.
Unfortunately, in exchange for obtaining a faster rate, the stopping criterion
{\bf (A4$'$)} now depends not only on $\gamma$, which can be computed,
but also on $\sigma$, which is hard to know in practice. Therefore stopping condition {\bf (A4$'$)} is not practical.

\subsection{Validity of assumption (A1)}
\label{sec:validity_assumption}
In this subsection, we discuss the validity of assumption {\bf
(A1)} and its relation to the assumptions used in \cite{Roc76a} and
\cite{KorBer76}. Roughly speaking, our assumption {\bf (A1)} is 
milder than the one used in \cite{Roc76a} and stronger than the one
used in \cite{KorBer76}.

First of all, assumption {\bf (A1)} is unnecessary for convergence in
terms of function value (Theorem~\ref{thm:convergence_f} and its
approximate version implied by Lemma~\ref{lem:key_approx}). Exponential
increase of the proximity parameter $\eta_t$ may sound restrictive, but
this is the setting we typically use in experiments (see \Secref{sec:impl}).
Assumption {\bf (A1)} is only necessary in Theorem~\ref{thm:convergence_w} and
Theorem~\ref{thm:convergence_w_approx} to translate the residual
function value $f(\vw^{t+1})-f(W^{\ast})$ into the residual distance
$\|\vw^{t+1}-W^{\ast}\|$.

We can roughly think of Assumption~{\bf (A1)} as a {\em local strong
convexity} assumption.  Here we say a function $f$ is {\em locally strongly
convex around} the set of minimizers $W^\ast$ if for all positive $C$, all the
points $\vw$ within distance $C$ from the set $W^{\ast}$, the objective
function $f$ is bounded below by a quadratic function, i.e., 
\begin{align}
\label{eq:local-strong-convex}
 f(\vw)-f(W^\ast)&\geq\sigma \|\vw-W^{\ast}\|^2\quad(\forall \vw:
 \|\vw-W^{\ast}\|\leq C),
\end{align}
where the positive constant $\sigma$ may depend on $C$. If the set of
minimizers $W^{\ast}$ is bounded, all the level sets of $f$
are bounded (see \citet[Theorem 8.7]{Roc70}). Therefore, if we make sure that the function
value $f(\vw^t)$ does not increase during the minimization, we can
assume that all points generated by DAL algorithm are contained
in some neighborhood around $W^{\ast}$  that contains the level set defined by the
initial function value $\{\vw\in\RR^n:f(\vw)\leq f(\vw^0)\}$; i.e.,
the local strong convexity of $f$ guarantees Assumption {\bf (A1)} with
$\alpha=2$.  

Note that \citet[p278]{KorBer76} used a slightly weaker assumption than the local strong
convexity \eqref{eq:local-strong-convex}; they assumed that there {\em
exists} a positive constant $C'>0$ such that
the local strong convexity \eqref{eq:local-strong-convex} is true for all
$\vw$ in the neighborhood $\|\vw-W^{\ast}\|\leq C'$ for some  $\sigma>0$. 

The local strong convexity~\eqref{eq:local-strong-convex} or Assumption
{\bf (A1)} fails when the objective function behaves like a constant
function around the set of minimizers $W^{\ast}$. In this case, DAL converges
rapidly in terms of function value due to
Theorem~\ref{thm:convergence_f}; however it does not  necessarily
converge in terms of the distance $\|\vw^t-W^{\ast}\|$. 

Note that the objective function $f$ is the sum of the loss term and the
regularization term. Even if the minimum eigenvalue of the Hessian of
the loss term is very close to zero, we can hope that the regularization
term holds the function up from the minimum objective value
$f(W^\ast)$. For example, when the loss term is {\em zero} and we only have the
$\ell_1$-regularization term $\phi_{\lambda}^{\ell_1}(\vw)$. The
objective $f(\vw)=\lambda\sum_{j=1}^n|w_j|$ can be lower-bounded as
\begin{align*}
 f(\vw)&\geq \frac{\lambda}{C}\|\vw\|^2\quad(\forall \vw:
 \|\vw\|\leq C),
\end{align*}
where the minimizer $\vw^{\ast}$ is $\vw^\ast=\bm{0}$. Note that the
$\ell_1$-regularizer is not (globally) strongly convex. The same
observation holds also for other regularizers we discussed in
\Secref{sec:instances}.

In the context of asymptotic analysis of AL algorithm,
 \cite{Roc76b} assumed that there exists $\tau>0$, such that in the ball
$\|\vbeta\|\leq \tau$ in $\RR^n$, the gradient of the convex conjugate
$f^\ast$ of the objective function $f$ is Lipschitz continuous
with constant $L$, i.e., 
\begin{align*}
\|\nabla f^{\ast}(\vbeta)-\nabla f^{\ast}(0)\|\leq L\|\vbeta\|.
\end{align*}
Note that because $\partial f (\nabla f^\ast(0))\ni
0$~\cite[Cor. 23.5.1]{Roc70}, $\nabla f^\ast(0)$ is the optimal solution
$\vw^{\ast}$ of \Eqref{eq:problem}, and it is unique by the continuity assumed above.

Our assumption {\bf (A1)} can be justified from Rockafellar's assumption
as follows.
\begin{theorem}
\label{thm:rockafellar_to_A1}
Rockafellar's assumption implies that the objective $f$ is locally
 strongly convex with $C=c\tau L$ and $\sigma=\min(1,(2c-1)/c^2)/(2L)$
for any positive constant $c$ ($\tau$ and $L$ are constants from
 Rockafellar's assumption).
\end{theorem}
\begin{proof}
The proof is a {\em local} version of the proof of Theorem X.4.2.2 in
 \citet{HirLem93} (Lipschitz continuity of $\nabla f^{\ast}$ implies
 strong convexity of $f$). See Appendix~\ref{sec:proof_rockafellar_to_A1}.
\end{proof}
 Note  that as the
constant $c$ that bounds the distance to the set of minimizers $W^{\ast}$ increases, the constant $\sigma$ becomes smaller
and the convergence guarantee in Theorem~\ref{thm:convergence_w} and 
\ref{thm:convergence_w_approx} become weaker (but still valid).

Nevertheless Assumption~{\bf (A1)} we use in Theorem~\ref{thm:convergence_w}
and \ref{thm:convergence_w_approx} are weaker than the local strong
convexity~\eqref{eq:local-strong-convex}, because we need
Assumption {\bf (A1)} to hold only on the points generated by DAL
algorithm. For example, if we only consider a finite number of steps, such
a constant $\sigma$ always exists.

Both assumptions in \cite{Roc76b} and \cite{KorBer76} are made for
asymptotic analysis. In fact, they require that as the optimization
proceeds, the solution becomes closer to the optimum $\vw^{\ast}$ in the
sense of the distance $\|\vw^t-W^{\ast}\|$ in \cite{KorBer76} and
$\|\vbeta\|$ in \cite{Roc76b}. However in both cases, it is hard to
predict how many  iterations it takes for the solution to be
sufficiently close to the optimum so that the super-linear convergence happens.

Our analysis is complementary to the above classical results. We have
shown that super-linear convergence happens {\em non-asymptotically}
under Assumption {\bf (A1)}, which is trivial for a
finite number of  steps, for DAL algorithm. Moreover, Assumption {\bf (A1)} can be justified for any
number of steps by the local strong convexity around $W^{\ast}$ \eqref{eq:local-strong-convex}.

\section{Previous studies}
\label{sec:algorithms}
In this section, we discuss
earlier studies in two categories. The first category comprises
methods that try to overcome the difficulty posed by the 
non-differentiability of the regularization term $\phi_\lambda(\vw)$.
The second category, which includes DAL algorithm in this paper, consists of methods that try to overcome the difficulty
posed by the coupling (or non-separability) introduced by the design
matrix $\mA$. The advantages and disadvantages of all the methods are
summarized in \Tabref{tab:comparison}.

\subsection{Constrained optimization, upper-bound minimization, and subgradient methods}
Many authors have focused on the {\em non-differentiability} of the
regularization term in order to efficiently minimize
\Eqref{eq:problem}. This view has lead to three types of approaches,
namely, (i) constrained optimization, (ii) upper-bound minimization, and
(iii) subgradient methods. 
  
In the constrained optimization
 approach, auxiliary variables are introduced to rewrite the
 non-differentiable regularization term as a linear function of
 conically-constrained auxiliary variables. For example, the
 $\ell_1$-norm of a vector $\vw$ can be rewritten as:
\begin{align*}
\|\vw\|_1=\sum_{j=1}^n\min_{w_j^{(+)},w_j^{(-)}\geq 0}\left(w_j^{(+)}+w_j^{(-)}\right)\quad{\rm s.t.}\quad  w_j=w_j^{(+)}-w_j^{(-)}, 
\end{align*}
 where $w_j^{(+)}$ and $w_j^{(-)}$ $(j=1,\ldots,n)$ are auxiliary variables and they
 are constrained in the positive-orthant cone. 
 Two major challenges of the auxiliary-variable formulation
are the increased size of the problem and the complexity of solving a
 constrained optimization problem. 

The projected gradient (PG) method (see \cite{Ber99}) iteratively
computes a gradient step and projects it back to the constraint-set.
The PG method in \cite{FigNowWri07} converges R-linearly\footnote{A
  sequence $\xi^t$ converges to $\xi$ R-linearly (R is for ``root'') if
  the residual 
  $|\xi^t-\xi|$ is bounded by a sequence $\epsilon^t$ that linearly
  converges to zero~\citep{NocWri99}.}, if the loss function is quadratic. However,
PG methods can be extremely slow when the design matrix $\mA$
is poorly conditioned. To overcome the scaling problem, the L-BFGS-B
algorithm~\citep{ByrLuNocZhu95} can be applied for the simple positive
orthant constraint that arises from the $\ell_1$ minimization. However, 
this approach does not easily extend to more general regularizers,
such as group lasso and trace norm regularization.

The interior-point (IP) method (see \cite{BoydBook}) is another
algorithm that is often used for constrained minimization; see 
\citet{KohKimBoy07,KimKohLusBoyGor07} for the application of IP methods to
sparse estimation problems.
Basically an IP method generates a sequence that approximately follows the
so called central path, which parametrically connects the analytic center of the
constraint-set and the optimal solution.
Although IP methods can tolerate poorly conditioned design matrices well,
 it is challenging to scale them up to very large dense problems.
The convergence of the IP method in \cite{KohKimBoy07} is empirically
  found to be linear.

The second approach (upper-bound minimization) constructs a
 differentiable upper-bound of the 
 non-differentiable regularization term. For example, the $\ell_1$-norm
 of a vector $\vw$ can be rewritten as follows:
\begin{align}
\label{eq:variational_upperbound}
 \|\vw\|_1 = \sum_{j=1}^n\min_{\alpha_j\geq 0}\left(\frac{w_j^2}{2\alpha_j} +\frac{\alpha_j}{2}\right).
\end{align}
In fact, the right-hand side is an upper bound of the left-hand side for
arbitrary non-negative $\alpha_j$ due to the inequality of arithmetic
and geometric means, and the equality is obtained by setting
$\alpha_j=|w_j|$.  
The advantage of the above parametric-upper-bound formulation is that
for a fixed set of $\alpha_j$, the problem~\eqref{eq:problemL1} becomes a
(weighted) quadratically regularized minimization problem, for which
various efficient algorithms already exist. The iteratively reweighted
shrinkage (IRS) method~\citep{GorRao97,Bio06,FigBioNow07} alternately
solves the quadratically regularized minimization problem and tightens
(re-weights) the upper-bound in \Eqref{eq:variational_upperbound}. A more general technique was studied in parallel
by the name of variational EM~\citep{Jaa97,Gir01,PalWipDreRao06}, which
generalizes the above upper-bound using Fenchel's
inequality~\citep{Roc70}. A similar approach that is based on Jensen's
inequality~\citep{Roc70} has been studied in the context of
multiple-kernel learning~\citep{MicPon05,RakBacCanGra08} and in the
context of multi-task learning~\citep{ArgEvgPon07,ArgMicPonYin08}.
The challenge in the IRS framework is the {\em
 singularity}~\citep{FigBioNow07} around the coordinate axis. For
 example, in the $\ell_1$-problem in \Eqref{eq:problemL1}, any zero
 component $w_j=0$ in the  initial vector $\vw$ will remain zero after
 any number of iterations. Moreover, it is possible to create a
 situation that the convergence becomes arbitrarily slow for finite
 $|w_j|$ because the convergence in the $\ell_1$ case is only
 linear~\citep{GorRao97}.  

\begin{table}[tb]
 \begin{center}
  \caption{Comparison of the algorithms to solve \Eqref{eq:problem}.
  In the columns, six
  methods, namely, projected gradient (PG), interior point (IP),
  iterative reweighted shrinkage (IRS), orthant-wise limited-memory
  quasi Newton (OWLQN), accelerated gradient (AG), and dual augmented
  Lagrangian (DAL), are categorized into four groups discussed in the
  text. The first row:  ``Poorly conditioned $\mA$'' means that a method
  can tolerate poorly 
  conditioned design matrices well. The second row: ``No singularity''
  means that a method does not suffer from singularity in the
  parametrization (see main text). The third row: ``Extensibility''
  means that a method can be easily extended beyond
  $\ell_1$-regularization. The forth row: ``Exploits sparsity of $\vw$''
  means that a method can exploit the sparsity in the intermediate
  solution. The fifth row: ``Efficient when'' indicates the situations
  each algorithm runs efficiently, namely, more samples than unknowns
  ($m\gg n$), more unknowns than samples ($m\ll n$), or does not matter (--).
The last row shows the rate of convergence known from literature.  The
  super-linear convergence of DAL is established in this paper.
  }
  \label{tab:comparison}
{\small \begin{tabular}{|c|c|c|c|c|c|c|c}
\hline
   &\multicolumn{2}{c|}{Constrained} & Upper-bound & Subgradient &
   \multicolumn{2}{c|}{Iterative}\\
&\multicolumn{2}{c|}{Optimization} & Minimization & Method &
\multicolumn{2}{c|}{Proximation}\\
\cline{2-7}
   & PG & IP & IRS & OWLQN & AG & DAL\\
\hline
Poorly conditioned $\mA$ &  --    &  \checkmark    &  \checkmark    & 
 \checkmark&  --  &\checkmark     \\
\hline
No singularity     &  \checkmark    & \checkmark     &  --    &  \checkmark&  \checkmark  &  \checkmark     \\
\hline
Extensibility      &  \checkmark    &  \checkmark    & \checkmark     & --  &  \checkmark    &  \checkmark  \\
\hline
Exploits sparsity of $\vw$ &  \checkmark     &   --      &  --   &  \checkmark     & \checkmark&  \checkmark \\
\hline
Efficient when &
-- & -- & -- & $m\gg n$ & $m\gg n$ & $m\ll n$\\
\hline
Convergence & ($O(e^{-k})$) & ($O(e^{-k})$) & $O(e^{-k})$  & ? & $O(1/k^2)$ & $o(e^{-k})$\\
\hline
  \end{tabular}}
\end{center}
\end{table}

The third approach (subgradient methods) directly handles the
non-differentiability through subgradients; see e.g., \cite{Ber99}.

A (stochastic) subgradient method typically converges as $O(1/\sqrt{k})$
for non-smooth problems in general and as $O(1/k)$ if the objective is strongly
convex; see \cite{ShaSinSre07,Lan10}. However, since the method is based on
gradients, it can easily fail when the problem is poorly
conditioned~(see e.g., \cite[Section 2.2]{YuVisGueSch10}). Therefore,
one of the challenges in subgradient-based approaches is to take the
second-order curvature information into account. This is especially
important to tackle large-scale problems with a possibly poorly
conditioned design matrix.
Orthant-wise limited memory quasi Newton (OWLQN, \cite{AndGao07}) and
subLBFGS~\citep{YuVisGueSch10} combine subgradients with the well known L-BFGS 
quasi Newton method~\citep{NocWri99}.
Although being very efficient for
$\ell_1$-regularization and piecewise linear loss functions, these methods depend on the efficiency of oracles that
 compute a descent direction and a step-size; therefore,
it is challenging to extend these methods to combinations of general
loss functions and general non-differentiable
regularizers. In addition, the convergence rates of the OWLQN and
subLBFGS methods are not known.


\subsection{Iterative proximation}
Yet another approach is to deal with the nondifferentiable
regularization through the proximity operation.
In fact, the proximity operator \eqref{eq:proximation} is easy to
compute for many practically relevant separable regularizers.

The remaining issue, therefore, is the coupling between variables
introduced by the design matrix $\mA$. We have shown in Sections \ref{sec:ist}
and \ref{sec:dalprox} that IST and DAL can be considered as two
different strategies to remove this coupling.

Recently many studies have focused on methods that iteratively
compute the proximal operation 
\eqref{eq:proximation}~\citep{FigNow03,DauDefMol04,ComWaj05,Nes07,BecTeb09,CaiCanShe08},
which can be described in an abstract manner as follows:
\begin{align}
\label{eq:proxupdate}
 \vw^{t+1}
&=\prox{\phi_{\lambda_{t}}}\left(\vy^{t}\right),
\end{align}
where $\prox{\phi_{\lambda}}$ is the proximal operator defined in
\Eqref{eq:proximation}. The above mentioned studies can be
differentiated by the different $\vy^{t}$ and $\lambda_t$ that they
use. 

For example, the IST approach (also known as the {\em forward-backward
splitting}~\citep{LioMer79,ComWaj05,DucSin09}) can be described as follows:
\begin{align*}
 \vy^t&:=\vw^t-\eta_t\mA\T\nabla f_{\ell}(\mA\vw^t),\\
\lambda_t&:=\lambda\eta_t.
\end{align*}
What we need to do at every iteration is only to compute the gradient at the
 current point, take a gradient step, and then perform the
proximal operation (\Eqref{eq:proxupdate}). Note that $\eta_t$ can be
considered as a step-size.

The IST method can be considered as a generalization of the projected
gradient method. Since the proximal gradient
step~\eqref{eq:istupdatepre} reduces to an 
ordinary gradient step when $\phi_{\lambda}=0$, the basic idea behind
IST is to keep the non-smooth term $\phi_{\lambda}$ as a part of the
proximity step~(see \cite{Lan10}). Consequently, the convergence
behaviour of IST is the same as that of (projected) gradient descent on the
differentiable loss term. Note that \cite{DucSin09} analyze the
case where the loss term is also non-differentiable in both batch and
online learning settings. \cite{LanLiZha09} also analyze the online
setting with a more general threshold operation.

IST  approach maintains sparsity of $\vw^t$ throughout the optimization,
which results in significant reduction of computational cost; this is
an advantage of  iterative proximation methods compared to
interior-point methods (e.g., \cite{KohKimBoy07}), because the solution
produced by interior-point methods becomes sparse only in an
asymptotic sense; see \cite{BoydBook}. 

The downside of the IST approach is the difficulty to choose the
step-size parameter $\eta_t$; this issue is especially
problematic when the design matrix $\mA$ is poorly conditioned.
In addition, the best known convergence rate of a
naive IST approach is $O(1/k)$~\citep{BecTeb09}, which means that the
number of iterations $k$ that we need to obtain a solution $\vw^k$ such
that $f(\vw^k)-f(\vw^\ast)\leq\epsilon$ grows linearly with
$1/\epsilon$, where $f(\vw^\ast)$ is the minimal value of
\Eqref{eq:problem}. 

SpaRSA~\citep{WriNowFig09} uses approximate second order curvature information 
for the selection of the step-size parameter
$\eta_t$. TwIST~\citep{BioFig07} is a
``two-step'' approach that tries to alleviate the poor efficiency of IST
when the design matrix is poorly conditioned. However the convergence
rates of SpaRSA and TwIST are unknown.

Accelerating strategies that use different choices of $\vy^t$ have been
proposed in \cite{Nes07} and \cite{BecTeb09} (denoted AG in
Tab.~\ref{tab:comparison}), which have $O(1/k^2)$ 
guarantee with almost the same computational cost per iteration; see
also \cite{Lan10}.

DAL can be considered as a new member of the family of iterative
proximation algorithms. We have qualitatively shown in \Secref{sec:dalprox}
 that DAL constructs a better lower bound of the loss term than
IST. Moreover, we have rigorously studied the convergence rate of
DAL and have shown that it converges super-linearly. Of course the fast
convergence of DAL comes with the increased cost per
iteration. Nevertheless, as we have qualitatively discussed in \Secref{sec:instances}, this increase is mild, because the sparsity of
intermediate solutions can be effectively exploited in the inner
minimization. We empirically compare DAL with other methods in \Secref{sec:results}.

There is of course an issue on how much one should precisely optimize
when the training error (plus the regularization term) is a crude
approximation of the generalization error~\citep{ShaSre08}. However the
reason we use sparse regularization is exactly that we are not only
interested in the predictive power. We argue that when we are using
sparse methods to gain insights into some problem, it is important that
we are sure that we are doing what we write in our paper (e.g., ``solve an
$\ell_1$-regularized minimization problem''), and someone else can
reliably recover the same sparsity pattern using any optimization
approach that employs some objective stopping criterion such as the duality
gap. Of course the stability of 
the optimal solution itself must be analyzed (see
\cite{ZhaYu06,MeiBue06}) and the trade-off between accuracy and sparsity
should be discussed. However, this is beyond the scope of this paper.

\section{Empirical results}
\label{sec:results}
In this section, we confirm the super-linear convergence of DAL
algorithm and compare it with other algorithms on $\ell_1$-regularized
logistic regression problems. The algorithms that we compare are
FISTA~\citep{BecTeb09}, OWLQN~\citep{AndGao07}, SpaRSA~\citep{WriNowFig09},
IRS~\citep{FigBioNow07}, and L1\_LOGREG~\citep{KohKimBoy07}. Note that
IST is not included because SpaRSA and FISTA are shown to clearly outperform the naive IST approach. We describe the logistic regression
problem and the implementation of all of the methods in
\Secref{sec:impl}. The 
synthetic experiments are presented in \Secref{sec:synth} and the
benchmark experiments are presented in \Secref{sec:bench}.

\subsection{Implementation}
\label{sec:impl}
In this subsection, we first describe the problem to be solved
and then explain the implementation of the above mentioned algorithms in
detail. 

For all algorithms except for IRS, the initial solution $\vw^0$ was set
to an all zero vector. For IRS, the initial solution was sampled from an
independent standard Gaussian distribution.

The CPU time was
measured on a Linux server with two 3.1 GHz Opteron Processors and 32GB of
RAM.

\subsubsection{$\ell_1$-regularized logistic regression}
\label{sec:lrl1}
 The logistic regression model is defined by the 
loss function
\begin{align}
\label{eq:loss_logistic}
 f_{\rm LR}(\vz)&=\sum_{i=1}^m\log(1+e^{-y_iz_i}),
\intertext{where $y_i\in\{-1,+1\}$ is a training label. The conjugate of
 the loss function can be obtained as follows:}
 f_{\rm LR}^\ast(-\valpha)&=\sum_{i=1}^m\left(\alpha_iy_i\log(\alpha_iy_i)+(1-\alpha_iy_i)\log(1-\alpha_iy_i)\right).\nonumber
\end{align}
Rewriting the dual problem \eqref{eq:dual_prob} we have the following
expression:
\begin{align}
\label{eq:L1_dual_obj}
 \maximize_{\valpha\in\RR^m} \qquad & -f_{\rm LR}^{\ast}(-\valpha),\\
\label{eq:L1_dual_const}
\subjectto \qquad & \|\mA\T\valpha\|_{\infty}\leq\lambda,
\end{align}
where $\|\vy\|_{\infty}=\max_{j=1,\ldots,n}|y_j|$ is the
$\ell_\infty$-norm; note that the
implicit constraint in \Eqref{eq:dual_prob} (through the indicator
function $\delta_{\lambda}^{\infty}$) is made explicit in \Eqref{eq:L1_dual_const}.

For the experiments in this section, we reparametrize the
regularization constant $\lambda$ as
$\lambda=\bar{\lambda}\|\mA\T\vy\|_{\infty}$. The reason for this
reparametrization is that for all $\bar{\lambda}\geq
0.5$ the solution $\vw$ can be shown to be zero;
thus we can measure the strength of the regularization relative to the
problem using $\bar{\lambda}$ instead of $\lambda$.
This is because the conjugate loss function $f_{\rm LR}^\ast$ takes the
minimum at $\alpha_i=y_i/2$ and the minimum is attained for
$\lambda\geq \|\mA\T(\vy/2)\|_{\infty}$ (see \Eqref{eq:L1_dual_const}).

\subsubsection{Duality gap}
\label{sec:gap}
We used the relative duality gap (RDG) as a stopping criterion with
tolerance $10^{-3}$. More specifically, we terminated all the algorithms
described below when RDG fell below $10^{-3}$. RDG was
computed as follows for all algorithms except L1\_LOGREG. For
L1\_LOGREG, we modified the stopping criterion implemented in the
original code by the authors from absolute duality gap to relative
duality gap. See also \cite{KohKimBoy07,WriNowFig09,TomSug09}. 

Let $\bar{\valpha}^t$ be any candidate dual vector at $t$th
iteration. For example, $\bar{\valpha}^t=\valpha^{t}$ for DAL and
$\bar{\valpha}^t=-\nabla f_{\ell}(\mA\vw^{t+1})$ for OWLQN, SpaRSA, and
IRS. Note that the above $\bar{\valpha}^t$ does not necessarily satisfy
the dual constraint \eqref{eq:L1_dual_const}. Thus we define
$\tilde{\valpha}^t=\bar{\valpha}^t\min(1,\lambda/\|\mA\T\bar{\valpha}^t\|_{\infty})$.
 Notice that $\|\mA\T\tilde{\valpha}^t\|_{\infty}\leq\lambda$ by
 construction. We compute the dual objective value as
$d(\vw^{t+1})=-f_{\ell}^{\ast}(-\tilde{\valpha}^t)$; see \Eqref{eq:L1_dual_obj}. Finally RDG$^{t+1}$
is obtained as RDG$^{t+1}=(f(\vw^{t+1})-d(\vw^{t+1}))/f(\vw^{t+1})$,
where $f$ is the primal objective function defined in \Eqref{eq:problem}.

The norm of the minimum norm subgradient is also frequently used
as a stopping criterion. However, there are two reasons for using
RDG instead. First, the gradient at the current point is not evaluated in
FISTA~\citep{BecTeb09}  and it requires additional computation, whereas
the vector $\valpha^t$ in the computation of RDG does not need to be the
gradient  at the current point; in fact the gradient at any point (or
any $m$-dimensional vector) gives a valid lower bound of the minimum
objective value. Second, 
since the gradient can change discontinuously at non differentiable
points, the norm of gradient does not reflect the distance from the
solution well; this is a problem for e.g., an interior-point method, 
because it produces a sparse solution only asymptotically.

\subsubsection{DAL}
DAL algorithm is implemented in MATLAB\footnote{The software is
available from \texttt{http://www.ibis.t.u-tokyo.ac.jp/ryotat/dal/}.}.
The inner minimization problem (see \Eqref{eq:dalL1inner}) is solved with 
Newton's method; we used the preconditioned conjugate gradient (PCG) method for
solving the associated Newton system (\texttt{pcg} function in
MATLAB); we use the diagonal elements of the Hessian matrix (see
\Eqref{eq:ALderiv2_l1}) as the preconditioner. The inner minimization is
terminated by the criterion \eqref{eq:dalstopcond} with  $\gamma=4$, because 
 the Hessian of the loss function \eqref{eq:loss_logistic} is
uniformly bounded by $1/4$ (see Table~\ref{tab:loss}).

 We chose the initial proximity parameter to be either
 $\eta_0=0.01/\lambda$ (conservative setting) or $\eta_0=1/\lambda$
 (aggressive setting) and increased $\eta_t$
 by a factor of 2 at every iteration.
Since $\eta_t$ appears in the soft-thresholding operation multiplied by
$\lambda$, it seems to be intuitive to choose $\eta_t$ inversely
proportional to $\lambda$ but we do not have a formal argument yet. 
We empirically discuss the choice of $\eta_0$ in more detail in
\Secref{sec:etachoice}. 

The algorithm was terminated when the RDG fell below $10^{-3}$.

\subsubsection{DAL-B}
\label{sec:dal-b}
DAL-B is a variant of DAL with an unregularized bias term (see update equations
 \eqref{eq:dalL1update_w_bias}-\eqref{eq:dalL1inner_bias}). 
This algorithm is included because L1\_LOGREG implemented by
 \cite{KohKimBoy07} estimates a bias term and therefore cannot be
 directly compared to DAL.

As an augmented Lagrangian algorithm, DAL-B solves the following dual
problem:
\begin{align}
\label{eq:L1-b_dual_obj}
 \maximize_{\valpha\in\RR^m} \qquad & -f_{\rm LR}^{\ast}(-\valpha)-\delta_{\lambda}^{\infty}(\vv),\\
\label{eq:L1-b_dual_const_1}
\subjectto \qquad & \mA\T\valpha=\vv,\\
\label{eq:L1-b_dual_const_2}
 &\bm{1}\T\valpha=0.
\end{align}
See also \Eqsref{eq:dual_prob} and \eqref{eq:dual_const}.

When implementing DAL-B, we noticed that sometimes the algorithm gets
stuck in a plateau where the additional equality constraint
\eqref{eq:L1-b_dual_const_2} improves very little. This was more likely
to happen when the condition of the design matrix was poor.

In order to avoid this undesirable slow-down, we heuristically adapt
the proximity parameter $\eta_t$ for the equality constraint
\eqref{eq:L1-b_dual_const_2}. Note that this kind of modification cannot
improve the theoretical convergence result without additional prior
information.
More specifically, we use proximity parameters $\eta_t^{(1)}$ and $\eta_t^{(2)}$
for equality constraints \eqref{eq:L1-b_dual_const_1} and
\eqref{eq:L1-b_dual_const_2}, respectively. The AL function
\eqref{eq:dalL1inner_bias} is rewritten as follows
\begin{align*}
 \valpha^t&=\argmin_{\valpha\in\RR^m}\left(
f_{\ell}^\ast(-\valpha)+\frac{1}{2\eta_t^{(1)}}\|\ST{\lambda\eta_t}(\vw^t+\eta_t\mA\T\valpha)\|^2+\frac{1}{2\eta_t^{(2)}}(b^t+\eta_t\bm{1}_m\T\valpha)^2\right).
\end{align*}
First we initialize $\eta_0^{(1)}=\eta_0^{(2)}=0.01/\lambda$
(conservative setting) or $\eta_0^{(1)}=\eta_0^{(2)}=1/\lambda$
(aggressive setting) as above. 
The proximity parameter $\eta_t^{(1)}$ with respect to
\Eqref{eq:L1-b_dual_const_1} is increased by the factor 2 at
every iteration (the same as DAL). The proximity parameter $\eta_t^{(2)}$
with respect to \Eqref{eq:L1-b_dual_const_2} is increased by a larger
factor 40 if the following conditions are satisfied: 
\begin{enumerate}
 \item The iteration counter $t>1$.
 \item The violation of the equality constraint
       \eqref{eq:L1-b_dual_const_2}, namely ${\rm
       viol}^t:=|\bm{1}\T\valpha^t|$, does not sufficiently decrease;
       i.e., ${\rm viol}^t>{\rm viol}^{t-1}/2$.
 \item The violation ${\rm viol}^t$ is larger than $10^{-3}$ (the
       tolerance of optimization).
\end{enumerate}
Otherwise, $\eta_t^{(2)}$ is increased by the same factor 2 as $\eta_t^{(1)}$.

Note that the theoretical results in \Secref{sec:analysis} still holds
if we replace $\eta_t$ in \Secref{sec:analysis} by $\eta_t^{(1)}$,
because $\eta_t^{(1)}\leq \eta_t^{(2)}$; i.e., the stopping
criterion~\eqref{eq:dalstopcond} and the convergence rates simply become
more conservative.

Since DAL-B has an additional equality
constraint~\eqref{eq:L1-b_dual_const_2}. We modified the computation of
relative duality gap described above by defining the candidate dual
vector $\bar{\valpha}^t$ as  $\bar{\valpha}^{t}=\valpha^{t}-\frac{1}{m}\bm{1}_m\bm{1}_m\T\valpha^t$. 

\subsubsection{Fast Iterative Shrinkage-Thresholding Algorithm (FISTA)}
FISTA algorithm~\citep{BecTeb09} is implemented in MATLAB. The
algorithm is terminated by the same RDG criterion except that the dual
objective is evaluated at $\vy^t$ in update equation
\eqref{eq:proxupdate} instead of $\vw^{t+1}$; this approach saves
unnecessary computation of gradients.

\subsubsection{Orthant-wise limited memory quasi Newton (OWLQN)}
OWLQN algorithm~\citep{AndGao07} is also implemented in MATLAB
because we found that our MATLAB implementation was faster than the C++
implementation provided by the authors; this is because MATLAB uses
optimized linear algebra routines while authors' implementation does
not. The algorithm is terminated by the same RDG criterion as DAL.

\subsubsection{Sparse Reconstruction by Separable Approximation (SpaRSA)}
SpaRSA algorithm~\citep{WriNowFig09} is implemented in MATLAB. We
modified the code provided by the
authors\footnote{\texttt{http://www.lx.it.pt/\~{}mtf/SpaRSA/}} to handle
the logistic loss function. The algorithm is terminated by the same RDG
criterion.

\subsubsection{Iterative reweighted shrinkage (IRS)}
IRS algorithm is implemented in MATLAB. At every iteration IRS
solves a ridge-regularized logistic regression problem with the regularizer
defined in \Eqref{eq:variational_upperbound}. This problem can be
converted into  a standard $\ell_2$-regularized logistic regression with
the design matrix $\tilde{\mA}=\mA{\rm
diag}(\sqrt{\alpha_1},\ldots,\sqrt{\alpha_n})$ by reparametrizing $w_j$ to
$\tilde{w}_j=w_j/\sqrt{\alpha_j}$. The
weight $\alpha_j$ is set to $|w_j^{t}|$ before solving the problem. Thus if
any $w_j^{t}=0$, the corresponding column of $\tilde{\mA}$ becomes zero
and it can be removed from the optimization. We use the limited memory
BFGS quasi-Newton method~\citep{NocWri99} to solve each sub-problem.

\subsubsection{Interior point algorithm (L1\_LOGREG)}
L1\_LOGREG algorithm~\citep{KohKimBoy07} is implemented in C. We
modified the code provided by the
authors\footnote{\texttt{http://www.stanford.edu/\~{}boyd/l1\_logreg/}} as a
C-MEX function so that it can be called directly from MATLAB without
saving matrices into files. We used the BLAS and LAPACK libraries
provided together with MATLAB R2008b (\texttt{-lmwblas} and
\texttt{-lmwlapack} options for the \texttt{mex} command). 
L1\_LOGREG is also terminated by the RDG criterion.

Note that L1\_LOGREG also estimates  an  unregularized bias term.
DAL algorithm with a bias term (DAL-B) is included to make the
comparison easy; see \Secref{sec:dal-b}.



\subsection{Synthetic experiment}
\label{sec:synth}
In this subsection, we first confirm the convergence behaviour of DAL
(\Secref{sec:validation}); second we compare the scaling of various algorithms
against the size of the problem (\Secref{sec:scaling}); finally we
discuss how to choose the initial proximity parameter $\eta_0$ (\Secref{sec:etachoice}).

\subsubsection{Experimental setting}
The elements of the design matrix $\mA\in\RR^{m\times n}$ were randomly
sampled from an independent standard Gaussian distribution. The true
classifier coefficient $\vbeta$ was generated by filling randomly chosen
element (4\%) of a $n$-dimensional vector with samples from an independent
standard Gaussian 
distribution; the remaining elements of the vector were set to zero. The
training label vector $\vy$ was obtained by taking the sign of
$\mA\vbeta+0.01\vxi$, where $\vxi\in\RR^m$ was a sample from
an $m$-dimensional independent standard Gaussian distribution. The whole
procedure was repeated ten times.

\subsubsection{Empirical validation of super-linear convergence}
\label{sec:validation}
In this section, we empirically confirm the validity of the convergence results
(Theorems~\ref{thm:convergence_f}, \ref{thm:convergence_w} and \ref{thm:convergence_w_approx})
obtained in the previous section and compare the efficiency of 
DAL, FISTA, OWLQN, SpaRSA, and IRS for the number of samples $m=1,024$
and the number of parameters $n=16,384$. L1\_LOGREG is not included
because it solves a different minimization problem. We use the
regularization constant $\bar{\lambda}=0.01$. For DAL, we used the
aggressive setting ($\eta_t=1/\lambda,2/\lambda,4/\lambda,\ldots$).

First in order to obtain the true minimizer\footnote{We assume that the
minimizer is unique.} $\vw^\ast$ of \Eqref{eq:problem},
we  ran DAL algorithm to obtain a solution with high precision (RDG $<10^{-9}$). Assuming that the support of this solution is correct, we
performed one Newton step of \Eqref{eq:problem} in the subspace of
active variables. The solution $\vw^\ast$ we obtained in this way
satisfied $\|\nabla f(\vw^\ast)\|<10^{-13}$, where $\nabla f(\vw^\ast)$
is the minimum norm subgradient of $f$ at $\vw^{\ast}$. The parameter
$\sigma$ in \Eqref{eq:lowerbounded} was estimated by taking the minimum
of $(f(\vw^t)-f(\vw^\ast))/\|\vw^t-\vw^\ast\|^2$ along the trajectory
obtained by the above minimization and multiplying the minimum value by
a safety factor of $0.7$. 
In order to estimate the residual norm $\|\vw^t-\vw^{\ast}\|$, we use
bounds \eqref{eq:thm1} and \eqref{eq:thm2} with $\alpha=2$ and the 
initial residual $\|\vw^0-\vw^{\ast}\|$. The bound
\eqref{eq:convergence_f} in Theorem~\ref{thm:convergence_f} is used with
the same initial residual to estimate the reduction in the function value.

\begin{figure}[tb]
 \begin{center}
  \hspace*{-20mm}\includegraphics[width=\textwidth]{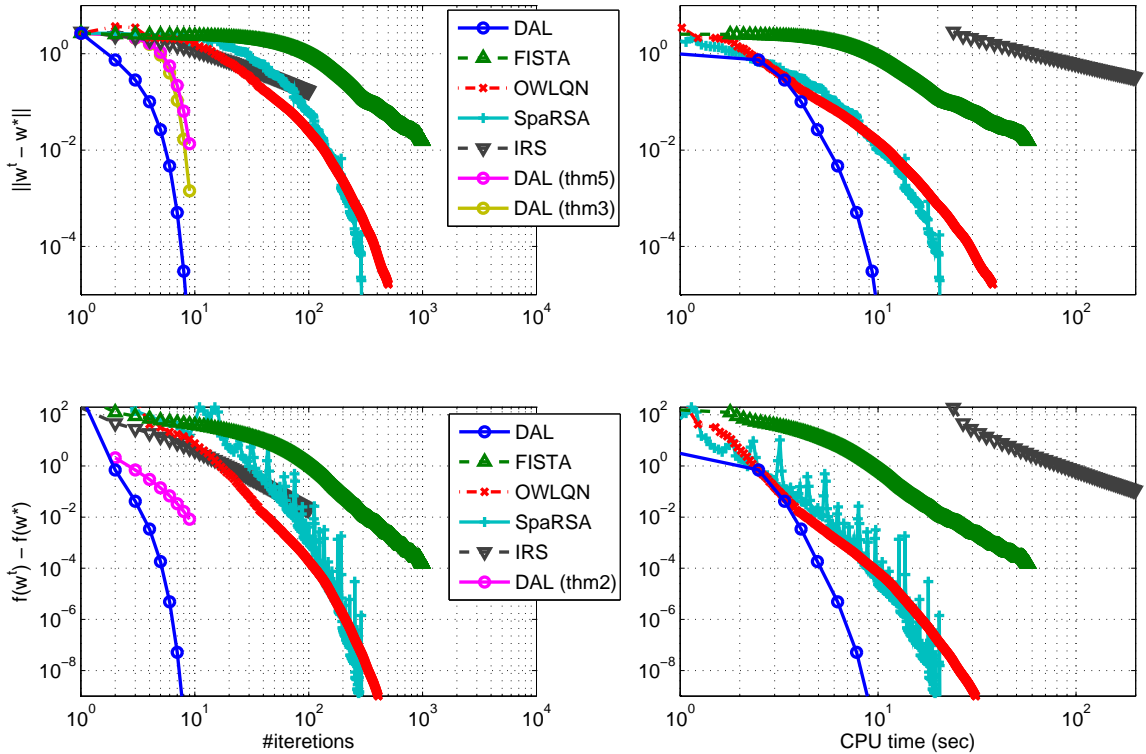}
  \caption{Empirical comparison of DAL,
  FISTA~\citep{BecTeb09}, OWLQN~\citep{AndGao07}, and
  SpaRSA~\citep{WriNowFig09}. Top left: residual norm   vs. number of 
  iterations. Also the theoretical guarantees for DAL from
  Theorems~\ref{thm:convergence_w} and \ref{thm:convergence_w_approx}
  are shown. Top right: residual norm vs. CPU time. Bottom left:
  residual in the function value vs. number of iterations. Bottom right:
  residual in the function value vs. CPU time.}
  \label{fig:iterations}
 \end{center}
\end{figure}

In \Figref{fig:iterations}, we show a result of a typical (single) run of
the algorithms described above. Note that the result is not averaged to
keep the meaning of theoretical bounds.

In the top left panel of \Figref{fig:iterations}, we can see that the
convergence in terms of the norm of the residual vector $\vw^t-\vw^\ast$
happens indeed rapidly as predicted by the theorems in
\Secref{sec:analysis}. The yellow curve shows the result of
Theorem~\ref{thm:convergence_w}, which assumes exact minimization of
\Eqref{eq:dalinner}, and the magenta curve shows the result of
Theorem~\ref{thm:convergence_w_approx}, which allows some error in the
minimization of \Eqref{eq:dalinner}. We can see that the difference 
between the optimistic analysis of Theorem~\ref{thm:convergence_w} and
the realistic analysis of Theorem~\ref{thm:convergence_w_approx} is negligible.
In this problem, in order to reach the quality of solution
DAL achieves in 10 iterations OWLQN and SpaRSA take at least 100
iterations and FISTA takes 1,000 iterations. The IRS approach required
about the same number of iterations as OWLQN and SpaRSA but each step was
much heavier than those two algorithms (see also the top right panel in
\Figref{fig:iterations}) and it was terminated after 100 iterations. 

The bottom left panel of \Figref{fig:iterations} shows comparison of five
algorithms DAL, FISTA, OWLQN, SpaRSA, and IRS in terms of the
decrease in the function value. Also plotted is the decrease in the
function value predicted by Theorem~\ref{thm:convergence_f} (magenta curve).
The convergence of DAL is the fastest also in
terms of function value. OWLQN and SpaRSA are the next after DAL and are
faster than FISTA. 

DAL needs to solve a minimization problem at
every iteration. Accordingly the operation required in each iteration is
heavier than those in FISTA, OWLQN, and SpaRSA. Thus we compare the
total CPU time spent by the 
algorithms in the right part of \Figref{fig:iterations}. It can be seen
that DAL can obtain a solution that is much more accurate in less than
10 seconds than the solution FISTA obtained after almost 60 seconds. In terms of
computation time, DAL and OWLQN seem to be on par at low
precision. However as the precision becomes higher DAL becomes clearly
faster than OWLQN. SpaRSA seems to be slightly slower than DAL and OWLQN.

\begin{figure}[tb]
 \begin{center}
  \includegraphics[width=\textwidth]{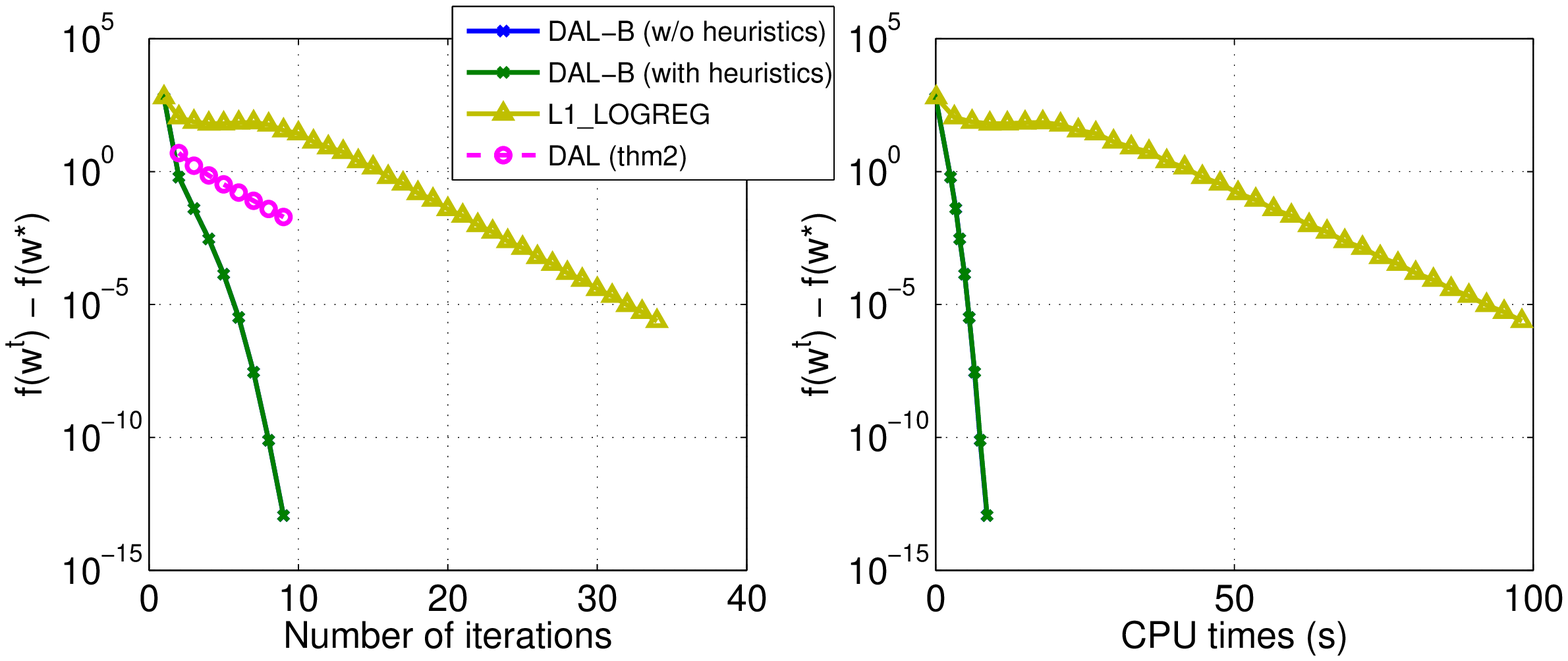}
  \caption{Comparison of DAL-B and L1\_LOGREG~\citep{KohKimBoy07}. Both
  algorithms estimate an unregularized bias term. The left panel shows
  the residual function value against the number of iterations. The
  right panel shows the same against the CPU time spent by the algorithms.}
  \label{fig:dal_l1logreg}
\end{center}
\end{figure}

Two algorithms (DAL-B and L1\_LOGREG) that also estimate an
unregularized bias term are compared in \Figref{fig:dal_l1logreg}. 
The number of observations $m=1,024$ and the number of parameters
$n=16,384$, and all other settings are identical as above.
A variant of DAL-B that does not use the heuristics described in
\Secref{sec:dal-b} is included for comparison.
For DAL-B without the heuristics, the proximity parameters $\eta_t^{(1)}$
and $\eta_t^{(2)}$ are both initialized to $1/\lambda$ and increased by
the factor 2. For DAL-B with the heuristics, the proximity parameter
$\eta_t^{(2)}$ is increased more aggressively; see \Secref{sec:dal-b}.

In the left panel in \Figref{fig:dal_l1logreg}, the residual of primal objective values of both
algorithms are plotted against the number of iterations. 
As empirically observed in \cite{KohKimBoy07}, L1\_LOGREG converges
linearly; after roughly 10 iterations, the residual function value
reduces by a factor around 2 in each iteration (a factor 1.85 was
reported in \cite{KohKimBoy07}).
The convergence of DAL-B is faster than L1\_LOGREG and the curve is
slightly concave downwards, which indicates the super-linearity of the
convergence. Note also that the linear convergence bound from
Theorem~\ref{thm:convergence_f} is shown. The heuristics described in
\Secref{sec:dal-b} shows almost no effect on this problem, probably because the
design matrix is well conditioned.

The right panel in \Figref{fig:dal_l1logreg} shows the same information
against the CPU time spent by the algorithms. DAL-B is roughly 10 times
faster than L1\_LOGREG to achieve residual less than $10^{-5}$.

\subsubsection{Scaling against the size of the problem}
\label{sec:scaling}
Here we compare how well different algorithms scale against the number
of parameters $n$. We fixed the number of samples $m$ at $m=1,024$
and varied the number of parameters from $n=4,096$ to
$n=524,288$. We used two regularization constants $\bar{\lambda}=0.1$
and $\bar{\lambda}=0.01$.

\begin{figure}[tb]
 \begin{center}
  \subfigure[$\bar{\lambda}=0.01$.]{
  \includegraphics[width=.5\textwidth]{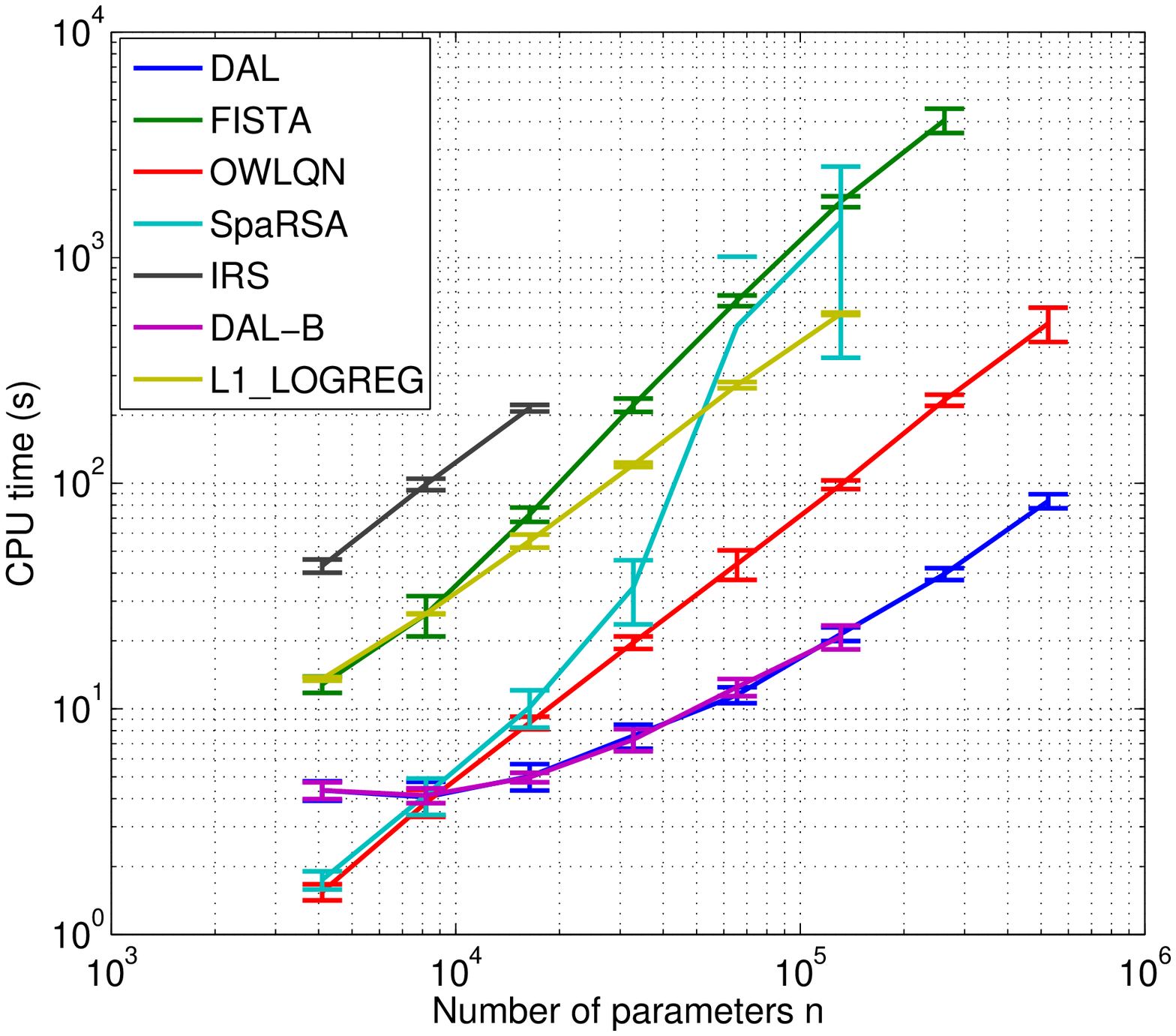}\label{fig:synthetic_0.01}}~\subfigure[$\bar{\lambda}=0.1$.]{
  \includegraphics[width=.5\textwidth]{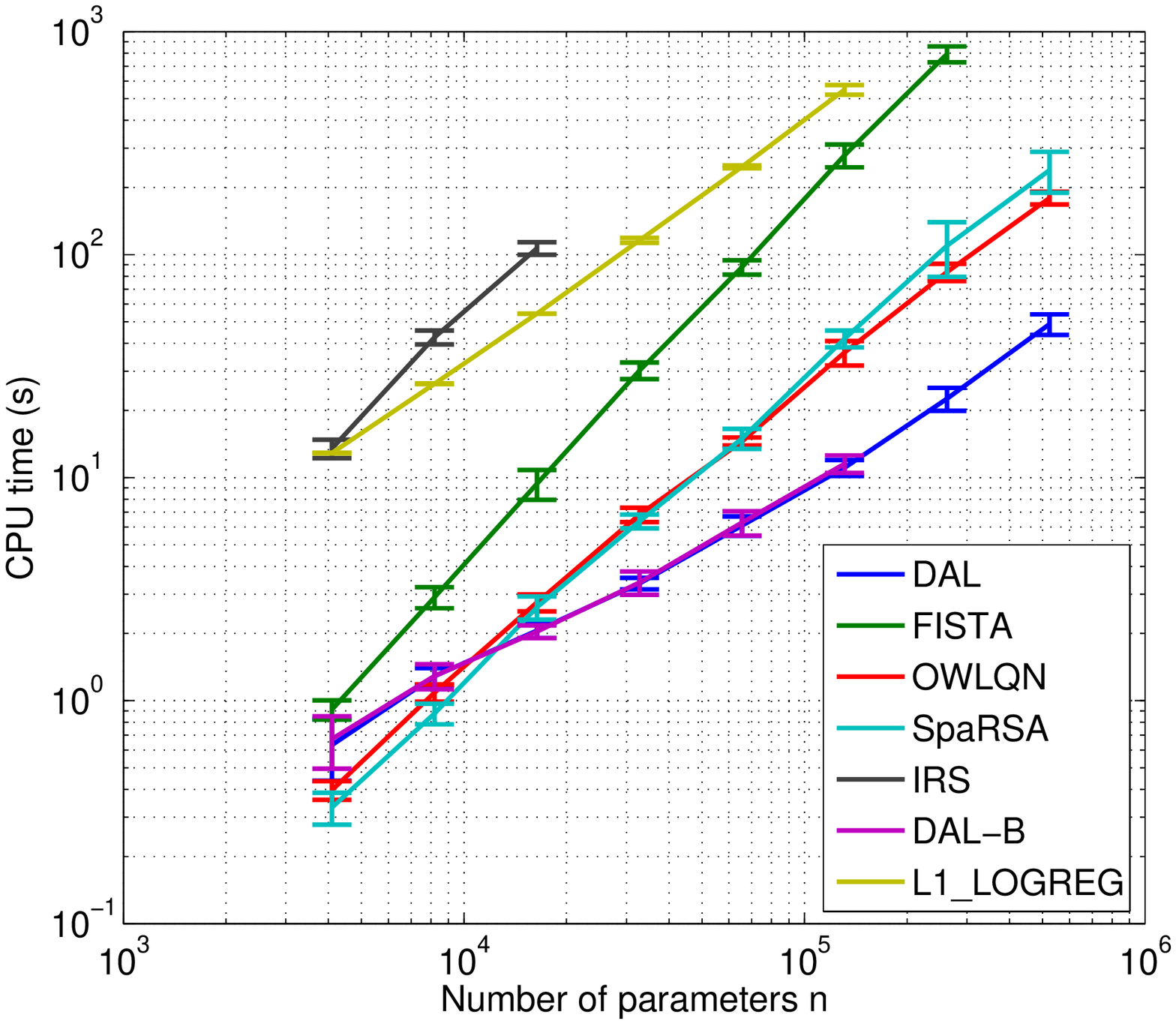}\label{fig:synthetic_0.1}}
   \end{center}
 \caption{CPU time of various algorithms on synthetic logistic
 regression problems.}
\label{fig:synthetic}
\end{figure}

The results are summarized in
\Figref{fig:synthetic}. Figures~\ref{fig:synthetic_0.01} and
\ref{fig:synthetic_0.1} show the results for $\bar{\lambda}=0.01$ and 
$\bar{\lambda}=0.1$, respectively. In each figure we plot the CPU time
spent to reach RDG$<10^{-3}$ against the number of parameters $n$.

One can see that DAL has the
mildest dependence on the number of parameters among the methods
compared. In particular, DAL is faster than other algorithms
for roughly $n> 10^{4}$. Also note that DAL and DAL-B show similar
scaling against the number of parameters; i.e., adding an unregularized
bias term has no significant influence on the computational efficiency.

For $\bar{\lambda}=0.01$, SpaRSA shows sharp increase in the CPU time
from around $n=32,\! 768$, which is similar to the result in
\cite{TomSug09} (Figure~3). Also notice the increased error-bar.
In fact, for $n\geq65,\! 536$, it had to be
stopped after 5,000 iterations in some runs, whereas it converged after
few hundred iterations in other runs.
On the other hand, SpaRSA scales similarly to
OWLQN and is more stable for $\bar{\lambda}=0.1$.

For all algorithms except L1\_LOGREG, solving the problem for larger
regularization constant $\bar{\lambda}=0.1$ requires less computation
than for $\bar{\lambda}=0.01$. Nevertheless the advantage of the DAL
algorithm is larger for the more computationally demanding situation of
$\bar{\lambda}=0.01$ against FISTA, OWLQN, SpaRSA, and IRS. On the other
hand, the advantage of DAL against L1\_LOGREG is larger for
$\bar{\lambda}=0.1$, because the CPU time of L1\_LOGREG is almost
constant in both cases. The CPU time of DAL with (DAL-B) and without
(DAL) the bias term are almost the same.

\subsubsection{Choice of $\eta_0$}
\label{sec:etachoice}
In this subsection, we show how the choice of the sequence $\eta_t$
changes the behaviour of DAL algorithm. We ran DAL algorithm for
$\bar{\lambda}=0.1$ with
$\eta_0=1/\lambda$ (as above), which we call the aggressive setting, and
$\eta_0=0.01/\lambda$, which we call the conservative setting. In both
cases, $\eta_t$ is increased by a factor of 2 as in the previous experiments.
No bias term is used.

In \Figref{fig:etabar}, plotted are the number of PCG steps for inner
minimization and the CPU time spent by DAL algorithm with the
conservative setting 
($\eta_0=0.01/\lambda$, left) and the aggressive setting
($\eta_0=1/\lambda$, right).
The average number of PCG steps and CPU time are shown as stacked bar-plots, in
which each segment of a bar corresponds to one outer iteration. One can see that in the conservative setting, DAL uses
roughly 8 to 10 outer iterations, whereas in the aggressive setting, the
number of outer iterations is reduced to less than a half (3 or 4). On
the other hand, the total number of PCG steps is only slightly smaller
in the aggressive setting. Therefore, in the aggressive setting DAL
spends more PCG steps for 
each outer iteration. It is worth noting that almost half of the PCG
iterations are spent for the first outer iteration in the aggressive setting,
whereas in the conservative setting the PCG steps are more
distributed. In terms of the CPU time, the aggressive setting is about
10--30\% faster than the conservative setting because it saves
both computation required for each outer-iteration and
inner-iteration. However, generally speaking increasing the proximity
parameter $\eta_t$ makes the condition of the problem worse; in fact we
found that the algorithm did not always converge for
$\eta_0=100/\lambda$. Thus it is not recommended to use  too large value
for $\eta_t$.

Figure~\ref{fig:etaloglog} compares the total CPU time spent by the two
variants of DAL. As discussed above, the aggressive setting
($\eta_0=1/\lambda$) is faster than the conservative setting
($\eta_0=0.01/\lambda$). However 
the difference is minor compared to the change in the proximity parameter
$\eta_0$,

\begin{figure}[tbp]
 \begin{center}
  \includegraphics[width=.7\textwidth]{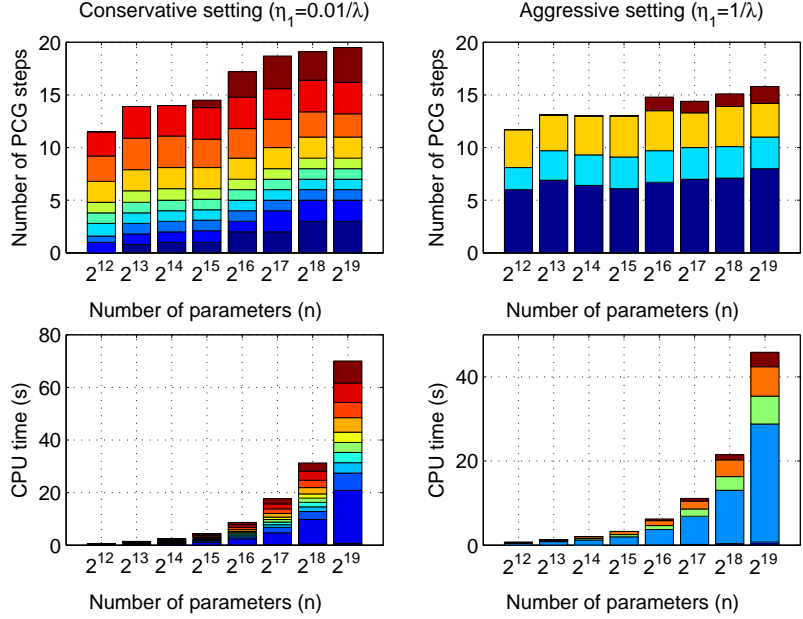}
  \caption{Comparison of behaviours of DAL algorithm for different
  choices of initial proximity parameter $\eta_0$. Left:
  $\eta_0=0.01/\lambda$ (conservative setting). Right:
  $\eta_0=1/\lambda$ (aggressive setting). On the top row, the
  cumulative numbers of PCG steps (inner steps) are shown. On the bottom
  row, the cumulative CPU time spent by the algorithm is shown.}
  \label{fig:etabar}
 \end{center}
\end{figure}

\begin{figure}[tbp]
 \begin{center}
  \includegraphics[width=.5\textwidth]{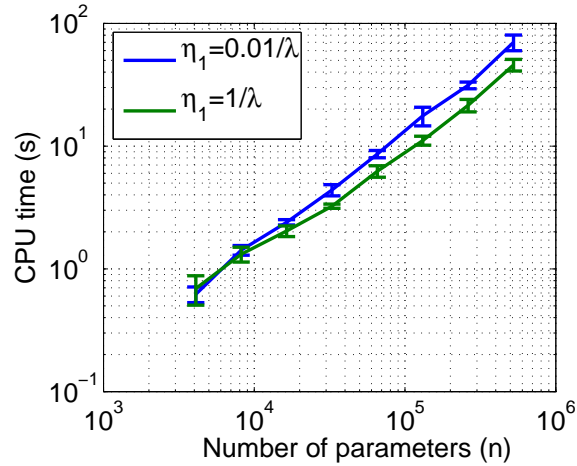}
  \caption{Comparison of conservative ($\eta_0=0.01/\lambda$) and
  aggressive ($\eta_0=1/\lambda$) choice of proximity parameter
  $\eta_0$. Note that the aggressive setting is used in
  Sections \ref{sec:validation} and \ref{sec:scaling} and the
  conservative setting is used in \Secref{sec:bench}.}
  \label{fig:etaloglog}
 \end{center}
\end{figure}

\subsection{Benchmark datasets}
\label{sec:bench}
In this subsection, we apply the algorithms discussed in the previous
subsection except IRS to benchmark datasets, and compare their efficiency
on various problems; IRS is omitted because it was clearly outperformed
by other methods on the synthetic data.

\subsubsection{Experimental setting}
The benchmark datasets we use are five datasets from NIPS 2003 Feature
Selection Challenge\footnote{The datasets are available from
\texttt{http://www.nipsfsc.ecs.soton.ac.uk/datasets/}; see
\cite{GuyGunNikZad06} for more information.}, 20 newsgroups dataset\footnote{The dataset is available from
\texttt{http://people.csail.mit.edu/jrennie/20Newsgroups/}.}, and 
a bioinformatics data\footnote{The data is available from
\texttt{http://www.plosbiology.org/article/info:doi/10.1371/journal.pbio.0030002}.}
provided by ~\cite{BarMouRioCaiStiVilWyaComGreSomMonOks04}

The five datasets from the Feature Selection Challenge ({\bf arcene,
dexter}, {\bf dorothea}, {\bf gisette}, and {\bf madelon})  are all
split into training-,  
validation-, and test-set. We combine the training- and validation-sets
and randomly split each dataset into a training-set that contains
two-thirds of the examples, and a
test-set that contains the remaining one-third. We apply the $\ell_1$-regularized logistic regression
solvers to the training-set and report the accuracy on the
test-set as well as the CPU time for training. This procedure was
repeated 10 times (also for the two other datasets below). The numbers of
training instances and features, and the format of each dataset (sparse or
dense) are summarized in \Tabref{tab:benchmark}. 

From the 20 newsgroups dataset ({\bf 20news}), we deal with the binary
classification of category ``alt.atheism'' vs. ``comp.graphics''. We use
the preprocessed MATLAB format data. The original dataset consists of
$1,061$ training examples and $707$ test examples. We again combine all
the examples and randomly split them into a training-set containing
two-thirds of the examples and a test-set containing the rest. 
The training example has $n=61,188$ features which are provided as a
sparse matrix. 

The goal in \cite{BarMouRioCaiStiVilWyaComGreSomMonOks04} is to predict
the response (good or poor) to recombinant human interferon beta
(rIFN$\beta$) treatment of multiple sclerosis patients from 
gene-expression measurements. The dataset is denoted as {\bf
gene}. The dataset consists of gene-expression
profile of 70 genes from 52 subjects. We again randomly select two-thirds
of the subjects for training and the remaining for testing. 
Following the setting in the original
paper, we used only the expression data from the beginning of the
treatment ($t=0$) and preprocessed the data by taking all the
polynomials up to third order, i.e., we compute (i) $x$, $x^2$, and $x^3$ for
each single feature $x$, (ii) $xy$, $x^2y$, and $xy^2$ for every pair of
features $(x,y)$, and (iii) $xyz$ for every triplet of features
$(x,y,z)$. As a result we obtain 62,195 $(=3\cdot 70+3\cdot
2,\! 415+54,\! 740)$ features.

In every dataset, we standardized each feature to zero mean and unit standard
deviation before 
applying the algorithms. Since the standardized design matrix
$\tilde{\mA}$ is usually dense even if the original matrix $\mA$ is
sparse, we provide function handles that compute $\tilde{\mA}\vx$ and
$\tilde{\mA}\T\vy$ instead of $\tilde{\mA}$ itself with DAL, FISTA,
OWLQN, and SpaRSA.  
 This can be 
done by keeping the vector of means and standard deviations of
the original design matrix as follows:
\begin{align*}
 \tilde{\mA}\vx &= \mA\mS^{-1}\vx-\bm{1}_m\bm{m}\T\mS^{-1}\vx,\\
 \tilde{\mA}\T\vy &= \mS^{-1}\mA\T(\vy-\frac{1}{m}\bm{1}_m\bm{1}_m\T\vy),
\end{align*}
where $\bm{m}\in\RR^n$ is the vector of means and $\mS$ is a $n\times n$
diagonal matrix that has the standard deviations of the original features
on the diagonal. If the standard deviation of any feature is zero, we
placed one in the corresponding element of $\mS$. L1\_LOGREG is
implemented with a similar technique; see \cite{KohKimBoy07}.

We compare the CPU time that is necessary to compute the whole
regularization path. In order to define the regularization path, we choose 20
log-linearly separated values from $\bar{\lambda}=0.5$ to
$\bar{\lambda}=0.001$, where $\bar{\lambda}$ is the normalized
regularization constant defined in \Secref{sec:lrl1}.
We apply a warm start strategy to all the
algorithms; i.e., we sequentially solve problems for smaller and
smaller regularization constants using the solution obtained from the
last optimization (for a larger regularization constant) as the initial
solution. 

All the methods were terminated when the relative duality gap
fell below $10^{-3}$. For DAL algorithms (DAL and DAL-B) we choose the
conservative setting, i.e.,  we initialize $\eta_0^{(1)}$ and
$\eta_0^{(2)}$ as $0.01/\lambda$.

\subsubsection{Results}

Table~\ref{tab:benchmark} summarizes the problems and the performance of
the algorithms. For each algorithm, we show the maximum test accuracy
obtained in the regularization path and the CPU time spent to compute
the whole path. The smallest and the second smallest CPU times are shown in
bold-face and italic, respectively. One can see that DAL is the fastest
in most cases when the number of parameters $n$ is larger than
the number of observations. In addition, the CPU time of two variants of
DAL (with and without the bias term) tend to be similar except {\bf
dorothea} dataset. For most datasets, the accuracy obtained by DAL
algorithm is close to FISTA, OWLQN, and SpaRSA, and the accuracy
obtained by DAL-B is close to L1\_LOGREG.

Figure~\ref{fig:dorothea} illustrates a typical situation where DAL
algorithm is efficient. Since the size of the inner minimization problem
 \eqref{eq:dalinner} is proportional to the number of observations $m$,
when $n\gg m$, DAL is more efficient than other methods that work in the
primal.

In contrast, \Figref{fig:madelon} illustrates the situation where DAL
is not very efficient compared to other algorithms. In
\Figref{fig:madelon}, we can also see that for all algorithms except
L1\_LOGREG, the cost of solving one minimization problem grows larger as
the regularization constant is reduced, whereas the cost seems almost
constant for L1\_LOGREG.

\begin{table}[tb]
\begin{center}
\caption{Results on benchmark datasets. We tested  six algorithms,
 namely, DAL, DAL-B, FISTA, OWLQN, SpaRSA, L1\_LOGREG on seven benchmark
 datasets. See main text for the
 description of the datasets. $m$ is the number of
 observations. $n$ is the number of features. For each algorithm, shown
 are the test accuracy and the CPU time spent to compute the
 regularization path with a warm-start strategy. All the numbers are
 averaged over 10 runs. Bold face numbers indicate the fastest CPU
 time. Italic numbers indicate CPU times that are within two times of
 the fastest CPU time.}
\label{tab:benchmark}
{\small 
\begin{tabular}{|c|c|r|r|r|r|r|r|r|}
\hline
\multicolumn{2}{|c|}{} &\textbf{arcene}&\textbf{dexter}&\textbf{dorothea}&\textbf{gisette}&\textbf{madelon}&\textbf{20news}&\textbf{gene}\\\hline
\multicolumn{2}{|c|}{\textbf{m}} &133 & 400 & 767 & 4667 & 1733 & 1179 & 35\\\hline
\multicolumn{2}{|c|}{\textbf{n}} &10000&20000&100000&5000&500&61188&62195\\\hline
 \multicolumn{2}{|c|}{\textbf{format}} & dense & sparse & sparse & dense &
		     dense & sparse & dense\\
\hline
\hline
\multirow{2}{1.5cm}{\textbf{DAL}} & accuracy &
70.60 & 91.75 & 93.71 & 97.70 & 61.53 & 92.84 & 82.35\\
\cline{2-9}
& time (s)&
{\bf 3.47} & {\em 4.20} & 36.61 & {\em 77.02} & 16.73 & {\em 28.10} & {\em 5.56}\\\hline
\hline
\multirow{2}{1.5cm}{\textbf{DAL-B}} &accuracy &
72.54 & 92.00 & 93.05 & 97.71 & 61.43 & 92.87 & 81.18\\
\cline{2-9}
& time (s)&
{\em 3.56} & {\em 4.77} & {\bf 10.60} & {\em 82.96} & 17.96 & {\em 26.31} & {\bf 5.49}\\\hline
\hline
\multirow{2}{1.5cm}{\textbf{FISTA}} & accuracy &
70.60 & 91.75 & 93.79 & 97.71 & 61.51 & 92.80 & 82.35\\
\cline{2-9}
& time (s)&
25.34 & 7.24 & 284.59 & {\bf 52.19} & {\em 10.40} & {\em 27.95} & 108.27\\\hline
\hline
\multirow{2}{1.5cm}{\textbf{OWLQN}} & accuracy &
70.60 & 91.75 & 93.76 & 97.70 & 61.56 & 92.82 & 82.35\\
\cline{2-9}
& time (s)&
17.63 & {\em 5.25} & 134.31 & {\em 70.96} & 19.08 & {\em 23.11} & 132.21\\\hline
\hline
\multirow{2}{1.5cm}{\textbf{SpaRSA}}& accuracy &
70.90 & 91.75 & 93.71 & 97.70 & 61.55 & 95.14 & 78.24\\
\cline{2-9}
& time (s)&
294.80 & 29.98 & 1377.20 & {\em 91.65} & {\em 10.11} & 310.96 & 1622.26\\\hline
\hline
\textbf{L1\_LOG-} & accuracy &
72.84 & 92.05 & 93.05 & 97.71 & 61.48 & 92.85 & 81.18\\
\cline{2-9}
\textbf{REG}& time (s)&
8.98 & {\bf 3.39} & 109.92 & {\em 98.37} & {\bf 5.90} & {\bf 21.48} & 16.58\\\hline
\end{tabular}}
\end{center}
\end{table}

\begin{figure}[tb]
 \begin{center}
  \includegraphics[width=.9\textwidth]{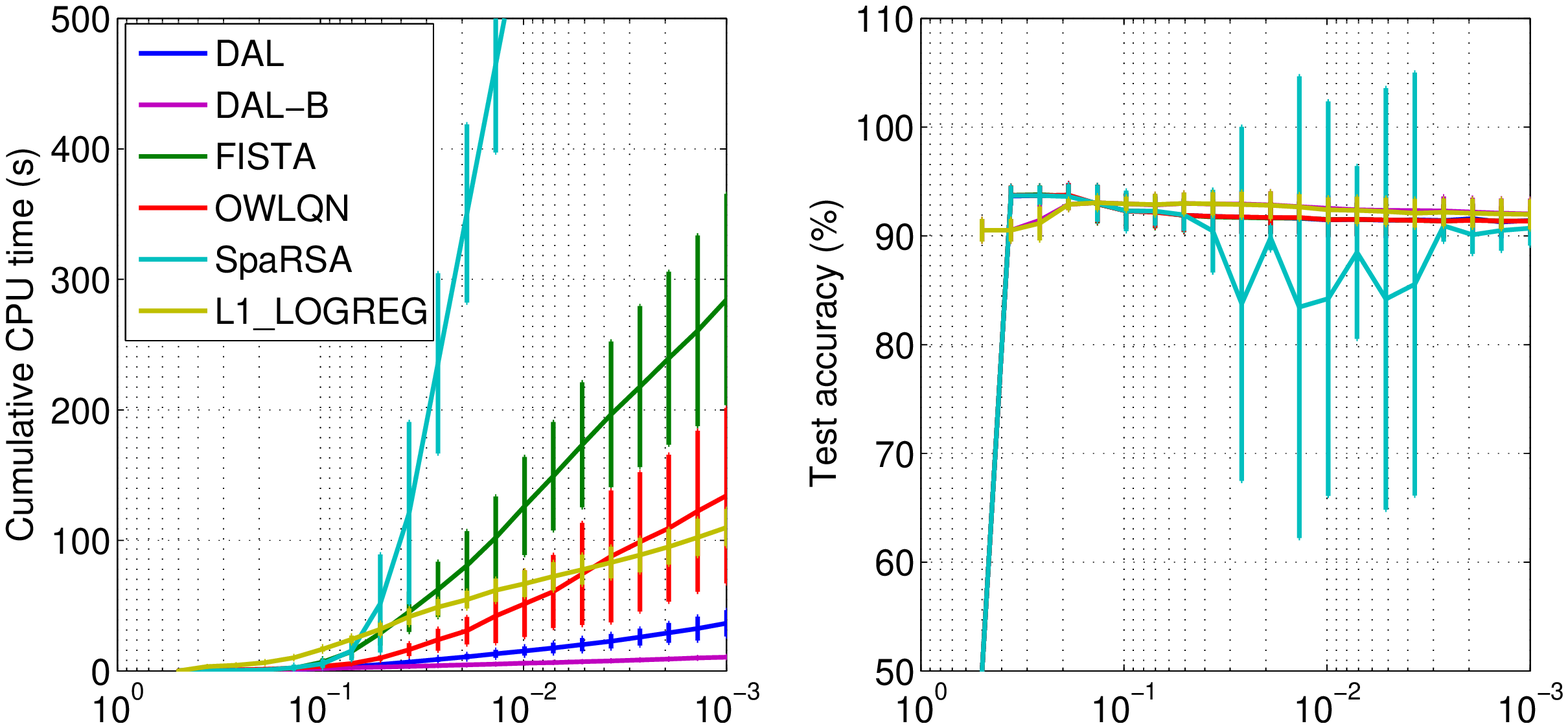}

\quad  {\small \sf Normalized regularization constant $\bar{\lambda}$}
\qquad\qquad  {\small \sf Normalized regularization constant $\bar{\lambda}$}
  \caption{{\bf Dorothea} dataset ($m=800$, $n=100,\! 000$). DAL is
  efficient in this case ($m\ll n$).}
  \label{fig:dorothea}
 \end{center}
\end{figure}

\begin{figure}[tb]
 \begin{center}
  \includegraphics[width=.9\textwidth]{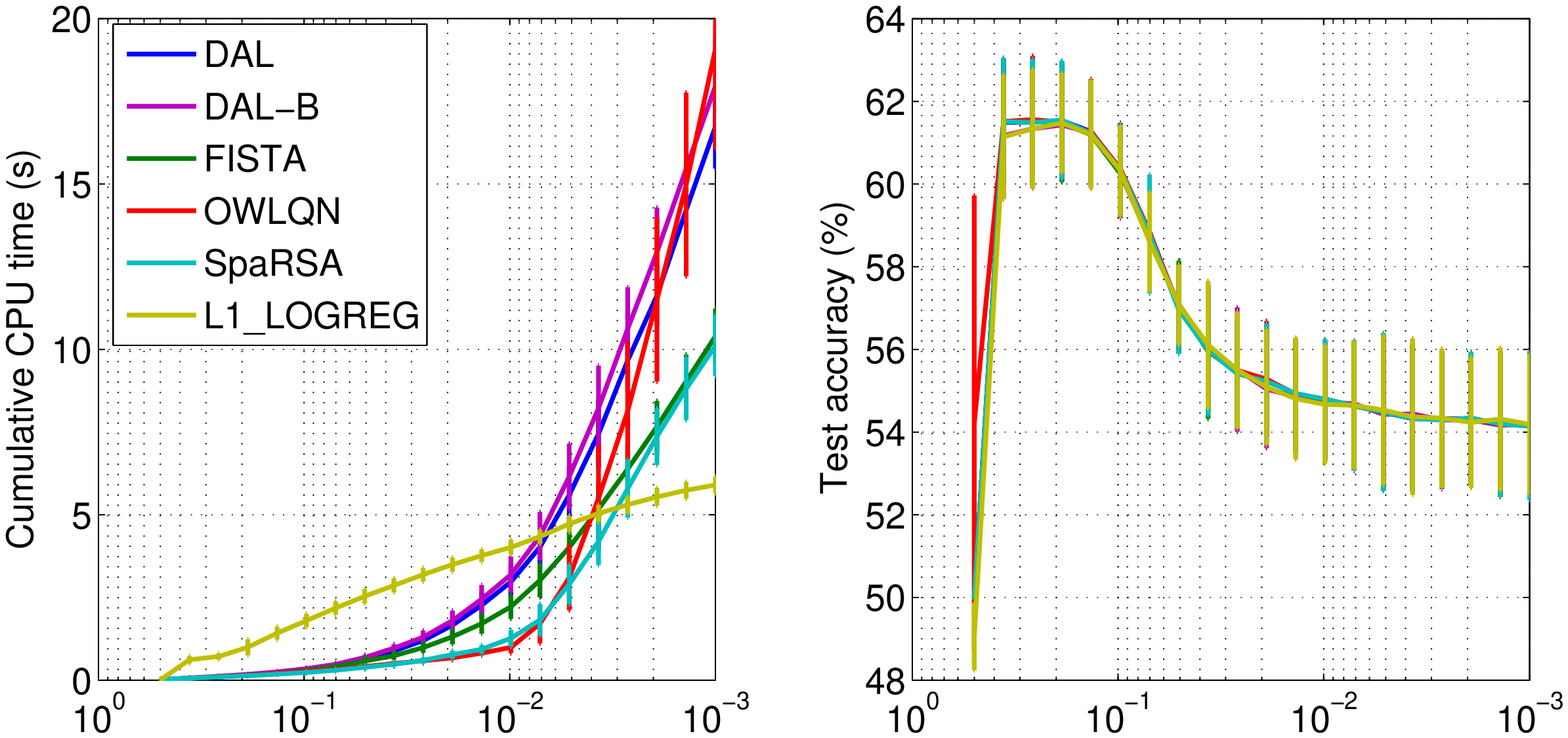}

\quad  {\small \sf Normalized regularization constant $\bar{\lambda}$}
\qquad\qquad  {\small \sf Normalized regularization constant $\bar{\lambda}$}
  \caption{{\bf Madelon} dataset ($m=2,\! 000$, $n=500$). DAL is not
  very efficient in this case ($m\gg n$).}
  \label{fig:madelon}
 \end{center}
\end{figure}

\section{Conclusion}
\label{sec:summary}
In this paper, we have extended DAL
algorithm~\citep{TomSug09} for general regularized minimization problems,
and provided it with a new view based on the
proximal minimization framework in \cite{Roc76b}. Generalizing the recent
result from \cite{BecTeb09}, we improved the convergence results on
super-linear convergence of augmented Lagrangian methods in
literature for the case of sparse estimation.

Importantly, most assumptions that we made in our analysis can be
checked independent of data.
Instead of assuming that the problem is strongly
convex we assume that the loss function has a Lipschitz continuous
gradient, which can be checked before receiving data. Another assumption
we have made is that the proximation with respect to the regularizer can
be computed analytically, which can also be checked without looking at data.
Moreover, we have shown that such assumption is valid for the
$\ell_1$-regularizer, group lasso regularizer, and any other support
function of some convex set for which the projection onto the set can be
analytically obtained.

 Compared to the general result
in \cite{Roc76b}, our result is stronger when the inner minimization
is solved approximately. Compared to
\cite{KorBer76}, we do not need to assume the strong convexity of the
objective function, which is
obviously violated for the dual of many sparsity regularized estimation
problems; instead we assume that the loss function has Lipschitz
continuous gradient. Note that we use no asymptotic arguments as in
\cite{Roc76b} and \cite{KorBer76}. Currently, our results does not
apply to primal-based augmented Lagrangian method discussed in
\cite{GolOsh08} for loss functions that are not strongly convex (e.g.,
logistic loss). The extension of our analysis to these methods is a
future work.

The theoretically predicted rapid convergence of DAL
algorithm is also empirically confirmed
in simulated $\ell_1$-regularized logistic regression problems.
Moreover, we have compared six recently proposed algorithms for
$\ell_1$-regularized logistic regression, namely DAL, FISTA, OWLQN,
SpaRSA, L1\_LOGREG, and IRS on synthetic and benchmark datasets.
On the synthetic datasets, we have shown that DAL has the mildest
dependence on the number of parameters among the methods
compared.
On the benchmark datasets, we have shown that DAL is the
fastest among the methods compared when the number of parameters is larger
than the number of observations on both sparse and dense datasets.

Furthermore, we have empirically investigated the relationship between
the choice of the initial proximity parameter $\eta_0$ and the number of
(inner/outer) iterations as well as the computation time. We found that
the computation can be sped up by choosing a large value for $\eta_0$;
however the improvement is often small compared to the change in
$\eta_0$ and choosing large value for $\eta_0$ can make the inner
minimization unstable by making the problem poorly conditioned.

There are basically two strategies to make an efficient optimization
algorithm. One is to use many iterations that are very light. FISTA,
SpaRSA, and OWLQN (and also stochastic approaches \citep{ShaSre08,DucSin09}) fall into this category. Theoretical convergence
guarantee is often weak for these methods, e.g., $O(1/k^2)$ for
FISTA. Another strategy is to use a small number of heavier
iterations. Interior point methods, such as L1\_LOGREG, are prominent
examples of this class. DAL can be considered as a member of the
second class. We have theoretically and empirically shown that DAL
requires a small number of outer iterations. At the same time, DAL
inherits good properties of iterative shrinkage/thresholding algorithms
from the first class. For example, it effectively uses the fact that the
proximal operation can be computed analytically, and it can maintain
the sparsity of the parameter vector during optimization.
Furthermore, we have shown that the dual formulation of DAL makes the
inner minimization efficient, because (i) typically the number of 
observations $m$ is smaller than the number of parameters~$n$, and (ii)
the gradient and Hessian of the inner objective can be computed efficiently
for sparse estimation problems.

Future work includes the extension of our analysis to the primal-based
augmented Lagrangian methods \citep{YinOshGolDar08,GolOsh08,YanZha09,LinCheWuMa09}, application
of approximate augmented Lagrangian methods and operator splitting
methods to machine learning problems (see \cite{TomSuzSug11}), and
application of DAL to more advanced sparse estimation problems (e.g.,
\cite{CaiCanShe08,WipNag08,TomHayKas10}). 

\paragraph{Acknowledgement}
We would like to thank Masakazu Kojima, Masao Fukushima, and Hisashi
Kashima for helpful discussions. This work was partially supported by
the Global COE program "The Research and Training Center for New
Development in Mathematics", MEXT KAKENHI 22700138, 22700289, and the
FIRST program. 

\appendix
\numberwithin{equation}{section}
\section{Preliminaries on proximal operation}
\label{sec:basics}
This section contains some basic results on proximal operation, which 
we use in later sections and is based on \cite{Mor65}, \cite{Roc70}, and
\cite{ComWaj05}. 

\subsection{Proximal operation}
Let $f$ be a closed proper convex function over $\RR^n$ that takes values in
$\RR\cup\{+\infty\}$. The {\em proximal operator}
with respect to $f$ is defined as follows:
\begin{align}
\label{eq:proximation_basic}
\prox{f}(\vz)=\argmin_{\vx\in\RR^n}\left(f(\vx)+\frac{1}{2}\|\vx-\vz\|^2\right).
\end{align}
Note that because of the strong convexity of the minimand in the right-hand side, the
above minimizer is unique.
Similarly we define the proximal operator with respect to the convex conjugate
function $f^\ast$ of $f$ as follows:
\begin{align*}
 \prox{f^\ast}(\vz)=\argmin_{\vx\in\RR^n}\left(f^\ast(\vx)+\frac{1}{2}\|\vx-\vz\|^2\right).
\end{align*}

The following elegant result is well known.
\begin{lemma}[Moreau's decomposition]
 \label{lem:moreau_decomp}
The proximation of a vector $\vz\in\RR^n$ with respect to a convex
 function $f$ and that with respect to its convex conjugate $f^\ast$ is
 complementary in the following sense:
\begin{align*}
 \prox{f}(\vz)+\prox{f^\ast}(\vz)&=\vz.
\end{align*}
\end{lemma}
\begin{proof}
 The proof can be found in  \cite{Mor65} and
 \citet[Thm.~31.5]{Roc70}. Here, we present a slightly more simple proof.

Let $\vx=\prox{f}(\vz)$ and $\vy=\prox{f^\ast}(\vz)$. By definition we have
$\partial f(\vx) + \vx-\vz\ni 0$ and $\partial f^\ast(\vy)+\vy-\vz\ni 0$,
 which imply
\begin{subequations}
\begin{align}
\label{eq:decomp1a}
 \partial f(\vx) \ni \vz-\vx,\\
\label{eq:decomp1b}
 \partial f^\ast(\vz-\vx)\ni\vx,
\end{align}
\end{subequations}
and
\begin{subequations}
 \begin{align}
\label{eq:decomp2a}
 \partial f^\ast(\vy)\ni\vz-\vy,\\
\label{eq:decomp2b}
 \partial f(\vz-\vy)\ni \vy,
 \end{align}
\end{subequations}
respectively, because $(\partial f)^{-1}=\partial f^\ast$~\cite[Cor. 23.5.1]{Roc70}.

From \Eqsref{eq:decomp1a} and \eqref{eq:decomp2a}, we have
\begin{align}
\label{eq:decomp1aa}
 f(\vz-\vy)&\geq f(\vx)+(\vz-\vy-\vx)\T(\vz-\vx),\\
 f^\ast(\vz-\vx)&\geq f^\ast(\vy)+(\vz-\vx-\vy)\T(\vz-\vy).
\end{align}
Similarly, \Eqsref{eq:decomp1b} and \eqref{eq:decomp2b} give
\begin{align}
 f^\ast(\vy)&\geq f^\ast(\vz-\vx)+(\vy-\vz+\vx)\T\vx,\\
\label{eq:decomp2bb}
 f(\vx)&\geq f(\vz-\vy)+(\vx-\vz+\vy)\T\vy.
\end{align}
Summing both sides of \Eqsref{eq:decomp1aa}--\eqref{eq:decomp2bb}, we have
\begin{align*}
 0\geq 2\|\vz-\vx-\vy\|^2,
\end{align*}
from which we conclude that $\vx+\vy=\vz$.
\end{proof}

Proximal operation can be considered as a generalization of the projection
onto a convex set. For example, if we take $f$ as the indicator function of
the $\ell_\infty$ ball of radius $\lambda$, i.e.,
$f(\vz)=\delta_\lambda^{\infty}(\vz)$ (see \Eqref{eq:indicatorL1}),
then the proximal operation with respect to $f$ is the
 projection onto the $\ell_\infty$-ball  \eqref{eq:clipping}.
 On the other hand, the proximal
 operation with respect to the $\ell_1$-regularizer is the soft-thresholding
 operator \eqref{eq:softth}. Therefore, we notice that
\begin{align*}
{\rm proj}_{[-\lambda,\lambda]}(\vz)+\prox{\lambda}^{\ell_1}(\vz)=\vz,
\end{align*}
which is a special case of Lemma~\ref{lem:moreau_decomp}, because the
$\ell_1$-regularizer is the convex conjugate of the indicator function
of the $\ell_\infty$-ball of radius $\lambda$; see \Figref{fig:envelopes}.

\subsection{Moreau's envelope}
The minimum attained in \Eqref{eq:proximation_basic} is called the
Moreau envelope of $f$:
\begin{align}
\label{eq:envelope_general}
 F(\vz) &=\min_{\vx\in\RR^n}\left(f(\vx)+\frac{1}{2}\|\vx-\vz\|^2\right).
\end{align}

The decomposition in Lemma~\ref{lem:moreau_decomp} can be expressed in
terms of envelope functions as follows.
\begin{lemma}[Decomposition and envelope functions]
\label{lem:decomp_envelope}
Let $f$ and $f^{\ast}$ be a pair of convex conjugate functions, and let $F$
 and $F^\ast$ be the Moreau envelopes of $f$ and $f^{\ast}$, respectively.
Then we have
\begin{align*}
 F(\vz)+F^\ast(\vz)&=\frac{1}{2}\|\vz\|^2.
\end{align*}
\end{lemma}

\begin{proof}
Let $\vx=\prox{f}(\vz)$ and $\vy=\prox{f^\ast}(\vz)$ as in the proof of
Lemma~\ref{lem:moreau_decomp}. From the definition of convex conjugate
 $f^{\ast}$, we have
\begin{align*}
 f(\vx)+f^{\ast}(\vy)&=\vy\T\vx,
\end{align*}
because $\vy=\vz-\vx\in\partial f(\vx)$~\cite[Theorem 23.5]{Roc70}.
Therefore, we have
\begin{align*}
F(\vz)+F^{\ast}(\vz)&=f(\vx)+\frac{1}{2}\|\vy\|^2+f^{\ast}(\vy)+\frac{1}{2}\|\vx\|^2\\
&=\vy\T\vx+\frac{1}{2}\|\vy\|^2+\frac{1}{2}\|\vx\|^2\\
&=\frac{1}{2}\|\vx+\vy\|^2=\frac{1}{2}\|\vz\|^2,
\end{align*}
where we used $\vx+\vy=\vz$ from 
 Lemma~\ref{lem:moreau_decomp} in the last line. 
\end{proof}
Note that $F^\ast$ is the Moreau envelope of $f^\ast$ and {\em not} the
convex conjugate of $F$.

\begin{figure}[tb]
 \begin{center}
  \includegraphics[width=.6\textwidth]{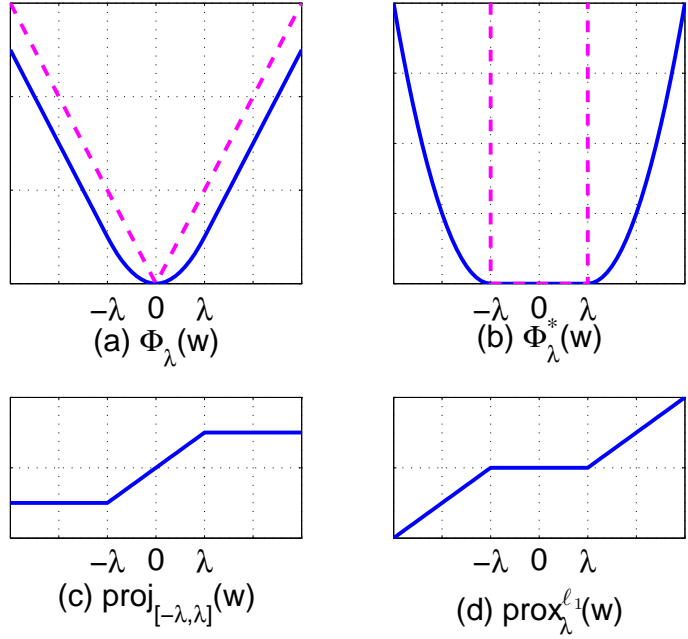}
  \caption{(a) The $\ell_1$-regularizer (dashed) and its envelope
  function $\Phi_{\lambda}$ (solid). (b) The indicator function
  $\delta_\lambda^\infty$ (dashed) and its envelope function
  $\Phi_{\lambda}^\ast$ (solid). (c) The
  derivative of $\Phi_{\lambda}$, which is the projection onto the
  interval $[-\lambda,\lambda]$; see~\Eqref{eq:clipping}. (d) The derivative of
  $\Phi_{\lambda}^{\ast}$, which is called the soft-threshold
  function~\eqref{eq:softth}. Note that the $\ell_1$-regularizer and the
  indicator function $\delta_{\lambda}^\infty$ are conjugate to each other.} 
  \label{fig:envelopes}
 \end{center}
\end{figure}

Moreau's envelope can be considered as a inf-convolution (see
\cite{Roc70}) of $f$ and a quadratic function $\|\cdot\|^2/2$. Accordingly
it is differentiable and the derivative is given in the following lemma.

\begin{lemma}[Derivative of Moreau's envelope]
\label{lem:moreau_deriv}
 Moreau's envelope function $F$ in
 \Eqref{eq:envelope_general} is continuously differentiable (even if
 $f$ is not differentiable) and the
 derivative can be written as follows:
\begin{align*}
 \nabla F(\vz) = \prox{f^\ast}(\vz).
\end{align*}
\end{lemma}
\begin{proof}
The proof can be found in \cite{Mor65} and \citet[Thm. 31.5]{Roc70}. We
 repeat the proof below for completeness. The proof consists of two
 parts. We first show that for all $\vz,\vz'\in\RR^n$
\begin{align}
\label{eq:lem8step1}
F(\vz') &\geq F(\vz)+(\vz'-\vz)\T\vy,
\intertext{where $\vy=\prox{f^\ast}(\vz)$, which implies that
 $\vy=\prox{f^\ast}(\vz)\in\partial F(\vz)$. Second, we show that}
\label{eq:lem8step2}
F(\vz') &\leq F(\vz)+(\vz'-\vz)\T\vy+\frac{\|\vz'-\vz\|^2}{2},
\end{align}
which implies the uniqueness of the subgradient of $F(\vz)$ for all $\vz$.

Inequality \eqref{eq:lem8step1} follows easily from the definition of
 the envelope function $F$ and Lemma~\ref{lem:moreau_decomp} as follows:
\begin{align*}
 F(\vz')-F(\vz)&=f(\vx')+\frac{1}{2}\|\vy'\|^2-f(\vx)-\frac{1}{2}\|\vy\|^2\\
&=\Bigl(f(\vx')-f(\vx)\Bigr)+\Bigl(\frac{1}{2}\|\vy'\|^2-\frac{1}{2}\|\vy\|^2\Bigr)\\
&\geq (\vx'-\vx)\T\vy+(\vy'-\vy)\T\vy\\
&=(\vz'-\vz)\T\vy,
\end{align*}
where $\vx=\prox{f}(\vz)$, $\vy=\prox{f^\ast}(\vz)$, and $\vx'$ and $\vy'$ are
 similarly defined. We used the convexity of $f$ with $\vy\in\partial
 f(\vx)$ and the convexity of $\|\cdot\|^2/2$ in the third line.

Second, we obtain inequality \eqref{eq:lem8step2} by upper-bounding
 $F(\vz')$ as follows:
\begin{align*}
 F(\vz')&=\min_{\vxi\in\RR^n}\Bigl(f(\vxi)+\frac{1}{2}\|\vxi-\vz'\|^2\Bigr)\\
 &\leq f(\vx)+\frac{1}{2}\|\vx-\vz'\|^2\\
 &= F(\vz)-\frac{1}{2}\|\vx-\vz\|^2+\frac{1}{2}\|\vx-\vz'\|^2\\
 &= F(\vz)-\frac{1}{2}\|\vx-\vz\|^2+\frac{1}{2}\|\vx-\vz+\vz-\vz'\|^2\\
&=F(\vz)+(\vx-\vz)\T(\vz-\vz')+\frac{1}{2}\|\vz'-\vz\|^2\\
&=F(\vz)+\vy\T(\vz'-\vz)+\frac{1}{2}\|\vz'-\vz\|^2.
\end{align*}
\end{proof}

The envelope functions of two convex functions $\phi_{\lambda}(w)=\lambda|w|$
and $\phi_{\lambda}^\ast(w)=\delta_{\lambda}^{\infty}(w)$, and their
derivatives (the projection \eqref{eq:clipping} and the
soft-threshold function \eqref{eq:softth}, respectively) are
schematically shown in \Figref{fig:envelopes}.


%
%
%


\section{Derivation of Equation~\eqref{eq:dalinner}}
\label{sec:deriv_dalinner}
\begin{align*}
\textrm{\Eqref{eq:maxmindal}}
&=\max_{\valpha\in\RR^m}\Biggl\{-f_{\ell}^\ast(-\valpha)+\min_{\vw\in\RR^n}\left(\phi_{\lambda}(\vw)+\frac{1}{2\eta_t}\|\vw-\vw^t-\eta_t\mA\T\valpha\|^2\right)\\
&\hspace{5cm}-\frac{1}{2\eta_t}\|\vw^t+\eta_t\mA\T\valpha\|^2\Biggr\}+\frac{\|\vw^t\|^2}{2\eta_t}\\
&=\max_{\valpha\in\RR^m}\left(-f_{\ell}^\ast(-\valpha)+\frac{1}{\eta_t}\Phi_{\lambda\eta_t}(\vw^{t}+\eta_t\mA\T\valpha)-\frac{1}{2\eta_t}\|\vw^t+\eta_t\mA\T\valpha\|^2\right)+\frac{\|\vw^t\|^2}{2\eta_t}\\
&=\max_{\valpha\in\RR^m}\left(-f_{\ell}^\ast(-\valpha)-\frac{1}{\eta_t}\Phi_{\lambda\eta_t}^{\ast}(\vw^{t}+\eta_t\mA\T\valpha)\right)+\frac{\|\vw^t\|^2}{2\eta_t},
\end{align*}
where we used the definition of the Moreau envelope in the second line and
Lemma~\ref{lem:decomp_envelope} in the third line. Finally omitting the
constant term $\|\vw^t\|^2/(2\eta_t)$ in the last line and reversing the
sign we obtain \Eqref{eq:dalinner}.

\section{Proofs}
\label{sec:proofs}
\subsection{Proof of Theorem~\ref{thm:convergence_w}}
\label{sec:proof_convergence_w}
\begin{proof}
The first step of the proof is to generalize  Lemma~\ref{lem:key_exact}
in two ways: first we allow a point $\vw^{\ast}_t$ in the set of
minimizers $W^\ast$ to be chosen for each time step, and second, we introduce a
parameter $\mu$ to tighten the bound. Let $\bar{\vw}^t$ be the closest
point from $\vw^{t}$ in $W^{\ast}$, namely
 $\bar{\vw}^t:=\argmin_{\vw^{\ast}\in W^{\ast}}\|\vw^{t}-\vw^{\ast}\|$.
From the proof of Lemma~\ref{lem:key_exact}, we have
\begin{align}
\eta_t(f(\bar{\vw}^{t+1})-f({\vw}^{t+1})) 
&\geq\dot{\bar{\vw}^{t+1}-\vw^{t+1}}{\vw^t-\vw^{t+1}}\nonumber\\
&=\dot{\bar{\vw}^{t+1}-\vw^{t+1}}{\vw^t-\bar{\vw}^{t}+\bar{\vw}^{t}-\bar{\vw}^{t+1}+\bar{\vw}^{t+1}-\vw^{t+1}}\nonumber\\
&=\|\bar{\vw}^{t+1}-\vw^{t+1}\|^2+\dot{\bar{\vw}^{t+1}-\vw^{t+1}}{\vw^t-\bar{\vw}^{t}}+\dot{\bar{\vw}^{t+1}-\vw^{t+1}}{\bar{\vw}^t-\bar{\vw}^{t+1}}\nonumber\\
&\geq \|\vw^{t+1}-W^{\ast}\|^2 -\|\vw^{t+1}-W^{\ast}\|\|\vw^{t}-W^{\ast}\| \nonumber\\
&\geq \|\vw^{t+1}-W^{\ast}\|^2-\left(\frac{\mu}{2}\|\vw^{t+1}-W^{\ast}\|^2+\frac{1}{2\mu}\|\vw^{t}-W^{\ast}\|^2\right)\quad(\forall\mu>0)\nonumber\\
&=\left(1-\frac{\mu}{2}\right)\|\vw^{t+1}-W^{\ast}\|^2-\frac{1}{2\mu}\|\vw^{t}-W^{\ast}\|^2,\qquad\qquad(\star)\nonumber
\end{align}
where the last inner product
 $\dot{\bar{\vw}^{t+1}-\vw^{t+1}}{\bar{\vw}^t-\bar{\vw}^{t+1}}$ in the
 third line is non-negative because the set of minimizers $W^{\ast}$ is
 a convex set, and  $\bar{\vw}^{t+1}$ is the projection
 of $\vw^{t+1}$  onto $W^{\ast}$; see \citet[Proposition
 B.11]{Ber99}. In addition, the fifth line follows from the inequality
 of arithmetic and geometric means.

Note that by setting $\mu=1$ in $(\star)$ and $\bar{\vw}^{t}=\bar{\vw}^{t+1}$, we recover Lemma~\ref{lem:key_exact}.
Now using assumption {\bf (A1)}, we obtain
 the following expression:
\begin{align*}
 \left(2\mu-\mu^2\right)\|{\vw}^{t+1}-W^{\ast}\|^2+2\mu \sigma\eta_{t}\|\vw^{t+1}-W^{\ast}\|^\alpha&\leq \|\vw^t-W^{\ast}\|^2.
\end{align*}
Maximizing the left hand side with respect to $\mu$, we have
 $\mu=1+\sigma\eta_{t}\|{\vw}^{t+1}-W^{\ast}\|^{\alpha-2}$ and accordingly,
\begin{align}
\left(1+\sigma\eta_{t}\|{\vw}^{t+1}-W^{\ast}\|^{\alpha-2}\right)^2\|{\vw}^{t+1}-W^{\ast}\|^2&\leq \|\vw^t-W^{\ast}\|^2.\nonumber
\intertext{Taking the square-root of both sides we obtain}
\label{eq:thm1-1}
\|{\vw}^{t+1}-W^{\ast}\| +\sigma\eta_{t} \|{\vw}^{t+1}-W^{\ast}\|^{\alpha-1}&\leq \|\vw^{t}-W^{\ast}\|.
\end{align}
The last part of the theorem is obtained by lower-bounding the left-hand
 side of the  above inequality using Young's inequality as follows:
\begin{align*}
&\|\vw^{t+1}-W^{\ast}\| +\sigma\eta_{t}\|\vw^{t+1}-W^{\ast}\|^{\alpha-1}\\
 &=(1+\sigma\eta_{t})\left(\frac{1}{1+\sigma\eta_{t}}\|\vw^{t+1}-W^{\ast}\|+\frac{\sigma\eta_{t}}{1+\sigma\eta_{t}}\|\vw^{t+1}-W^{\ast}\|^{\alpha-1}\right)\\
&\geq(1+\sigma\eta_{t})\|\vw^{t+1}-W^{\ast}\|^{\frac{1}{1+\sigma\eta_{t}}}\cdot \|\vw^{t+1}-W^{\ast}\|^{\frac{(\alpha-1)\sigma\eta_{t}}{1+\sigma\eta_{t}}}\\
&=(1+\sigma\eta_{t})\|\vw^{t+1}-W^{\ast}\|^{\frac{1+(\alpha-1)\sigma\eta_{t}}{1+\sigma\eta_{t}}}.
\end{align*}
Substituting this relation back into inequality \eqref{eq:thm1-1}
completes the proof of the theorem.
\end{proof}

\subsection{Proof of Lemma~\ref{lem:key_approx}}
\label{sec:proof_key_approx}
\begin{proof}
First let us define $\vdelta^t\in\RR^m$ as the gradient of the AL
 function \eqref{eq:ALderiv1} at the approximate minimizer $\valpha^t$ follows:
\begin{align*}
 \vdelta^t:=\nabla \varphi_t(\valpha^t)=-\nabla f_{\ell}^\ast(-\valpha^t)+\mA\vw^{t+1},
\end{align*}
where $\vw^{t+1}:=\prox{\phi_{\lambda\eta_t}}(\vw^t+\eta_t\mA\T\valpha^t)$. Note that
 $\|\vdelta^t\|\leq\sqrt{\frac{\gamma}{\eta_t}}\|\vw^{t+1}-\vw^{t}\|$ from
assumption~{\bf (A4)}. Using  Cor. 23.5.1 from \cite{Roc70}, we have
\begin{align}
\label{eq:proof2_loss_grad}
\nabla f_{\ell}(\mA\vw^{t+1}-\vdelta^t)=\nabla f_{\ell}\left(\nabla f_{\ell}^\ast(-\valpha^t)\right)=-\valpha^t,
\end{align}
which implies that if $\vdelta^t$ is small, $-\valpha^t$ is
approximately the  gradient of  the loss term at $\vw^{t+1}$.

Moreover, let $\vq^t=\vw^t+\eta_t\mA\T\valpha^t$. Since
 $\vw^{t+1}=\prox{\phi_{\lambda\eta_t}}(\vq^t)$ (assumption {\bf (A3)}), we have 
\begin{align}
&\partial\phi_{\lambda\eta_t}(\vw^{t+1})+(\vw^{t+1}-\vq^t)\ni 0,\nonumber
\intertext{which implies}
\label{eq:proof2_prox_subgrad}
&(\vq^t-\vw^{t+1})/\eta_t\in\partial \phi_{\lambda}(\vw^{t+1}),
\end{align}
 because $\phi_{\lambda\eta_t}=\eta_t\phi_{\lambda}$.

Now we are ready to derive an analogue of inequality
 \eqref{eq:proof1_convexity}  in the proof of
 Lemma~\ref{lem:key_exact}. Let $\vw\in\RR^n$ be an arbitrary vector. We can
 decompose the residual value in the left hand side of
 inequality~\eqref{eq:proof1_convexity} as follows:
\begin{align*}
 \eta_t(f(\vw)-f(\vw^{t+1}))&=\eta_t(\underbrace{f_{\ell}(\mA\vw)-f_{\ell}(\mA\vw^{t+1}-\vdelta^t)}_{\rm
 (A)})\\
&\qquad+\eta_t(\underbrace{f_{\ell}(\mA\vw^{t+1}-\vdelta^t)-f_{\ell}(\mA\vw^{t+1})}_{\rm
 (B)})\\
&\qquad+\eta_t(\underbrace{\phi_{\lambda}(\vw)-\phi_{\lambda}(\vw^{t+1})}_{\rm
 (C)}).
\end{align*}
The above terms (A), (B), and (C) can be separately bounded using the
 convexity of $f_{\ell}$ and $\phi_\lambda$ as follows:
\begin{flalign*}
{\rm (A)}: && f_{\ell}(\mA\vw)-f_{\ell}(\mA\vw^{t+1}-\vdelta^t)
 &\geq\dot{\mA(\vw-\vw^{t+1})+\vdelta^t}{-\valpha^t},&&\\
{\rm (B)}:  &&
 f_{\ell}(\mA\vw^{t+1}-\vdelta^t)-f_{\ell}(\mA\vw^{t+1})&\geq-\dot{\vdelta^t}{-\valpha^t}-\frac{1}{2\gamma}\|\vdelta^t\|^2,&&\\
 {\rm (C)}: && \phi_{\lambda}(\vw)-\phi_{\lambda}(\vw^{t+1})&\geq
\dot{\vw-\vw^{t+1}}{\frac{\vw^t+\eta_t\mA\T\valpha^t-\vw^{t+1}}{\eta_t}},&&
\end{flalign*}
where (A) is due to \Eqref{eq:proof2_loss_grad}, (B) is due to
 assumption {\bf (A2)} and \citet[Theorem X.4.2.2]{HirLem93},
 and (C) is due to \Eqref{eq:proof2_prox_subgrad}.
Combining (A), (B), and (C), we have the following expression:
\begin{align*}
 \eta_t(f(\vw)-f(\vw^{t+1}))&\geq \dot{\vw^t-\vw^{t+1}}{\vw-\vw^{t+1}}-\frac{\eta_t}{2\gamma}\|\vdelta^t\|^2.
\end{align*}
Note that the above inequality reduces to
 \Eqref{eq:proof1_convexity} if $\|\vdelta_t\|=0$ (exact minimization). Using
 assumption {\bf (A4)}, we obtain 
\begin{align*}
\eta_t(f(\vw)-f(\vw^{t+1}))&\geq\dot{\vw^t-\vw+\vw-\vw^{t+1}}{\vw-\vw^{t+1}}-\frac{1}{2}\|\vw^t-\vw^{t+1}\|^2\\
&=\frac{1}{2}\|\vw-\vw^{t+1}\|^2-\frac{1}{2}\|\vw-\vw^{t}\|^2,
\end{align*}
which completes the proof.
\end{proof}

\subsection{Proof of Theorem~\ref{thm:fasterrate}}
\label{sec:proof_fasterrate}
\begin{proof}
Let $\delta=(1-\epsilon)/(\sigma\eta_t)$. Note that $\delta\leq 3/4<1$
 from the assumption.
Following the proof of Lemma~\ref{lem:key_approx} with $\vw=\bar{\vw}^{t}$, we have
\begin{align*}
& \eta_t(f(\bar{\vw}^{t+1})-f(\vw^{t+1}))\\
&= \eta_t(f(\bar{\vw}^{t})-f(\vw^{t+1}))\\
&\geq\dot{\vw^t-\vw^{t+1}}{\bar{\vw}^{t}-\vw^{t+1}}-\frac{\delta}{2}\|\vw^{t+1}-\vw^{t}\|^2\\
&\geq\dot{\vw^t-\bar{\vw}^{t}+\bar{\vw}^{t}-\vw^{t+1}}{\bar{\vw}^{t}-\vw^{t+1}}-\frac{1}{2}\|\vw^{t+1}-\vw^{t}\|^2+\frac{1-\delta}{2}\|\vw^{t+1}-\vw^{t}\|^2\\
&=\frac{1}{2}\|\bar{\vw}^{t}-\vw^{t+1}\|^2-\frac{1}{2}\|\bar{\vw}^{t}-\vw^{t}\|^2+\frac{1-\delta}{2}\|\vw^{t+1}-\vw^{t}\|^2\\
 &\geq
 \frac{1}{2}\|\vw^{t+1}-W^{\ast}\|^2-\frac{1}{2}\|\vw^{t}-W^{\ast}\|^2\\
&\qquad\qquad+\frac{1-\delta}{2}\left(\|\vw^{t+1}-W^{\ast}\|^2+\|\vw^{t}-W^{\ast}\|^2+2\dot{\vw^{t+1}-\bar{\vw}^{t+1}}{\bar{\vw}^t-\vw^{t}}\right)\\
 &\geq -(1-\delta)\|\vw^{t+1}-W^{\ast}\|\|\vw^{t}-W^{\ast}\|+\left(1-\frac{\delta}{2}\right)\|\vw^{t+1}-W^{\ast}\|^2-\frac{\delta}{2}\|\vw^{t}-W^{\ast}\|^2\\
&\geq
 -(1-\delta)\left(\frac{1}{2\mu}\|\vw^{t}-W^{\ast}\|^2+\frac{\mu}{2}\|\vw^{t+1}-W^{\ast}\|^2\right)\\
&\qquad\qquad\qquad+\left(1-\frac{\delta}{2}\right)\|\vw^{t+1}-W^{\ast}\|^2-\frac{\delta}{2}\|\vw^{t}-W^{\ast}\|^2
 \quad (\forall \mu>0),
\end{align*}
where we used $\|\bar{\vw}^{t+1}-\bar{\vw}^{t}\|^2\geq 0$,
 $\dot{\vw^{t+1}-\bar{\vw}^{t+1}}{\bar{\vw}^{t+1}-\bar{\vw}^{t}}\geq 0$
 and $\dot{\bar{\vw}^{t+1}-\bar{\vw}^{t}}{\bar{\vw}^{t}-\vw^t}\geq 0$ in
 the sixth line; the seventh line
 follows from Cauchy-Schwartz inequality; the eighth line follows
 from the inequality of arithmetic and geometric means. 

Applying assumption {\bf (A1)} with $\alpha=2$ to the above expression,
 we have
\begin{align*}
\frac{1-\delta}{2\mu}\|\vw^{t}-W^{\ast}\|^2&\geq -\frac{1-\delta}{2}\mu \|\vw^{t+1}-W^{\ast}\|^2
+\left({\textstyle \left(1-\frac{\delta}{2}\right)}+\sigma\eta_t\right)\|\vw^{t+1}-W^{\ast}\|^2-\frac{\delta}{2}\|\vw^{t}-W^{\ast}\|^2.
\end{align*}
Multiplying both sides with $\mu/\|\vw^{t+1}-W^{\ast}\|^2$, we have
\begin{align}
\label{eq:proof_fastrate_quad}
\frac{1-\delta}{2}\frac{\|\vw^{t}-W^{\ast}\|^2}{\|\vw^{t+1}-W^{\ast}\|^2}&\geq -\frac{1-\delta}{2}\mu^2+\left\{\left(1-\frac{\delta}{2}\right)+\sigma\eta_t-\frac{\delta}{2}\frac{\|\vw^{t}-W^{\ast}\|^2}{\|\vw^{t+1}-W^{\ast}\|^2}\right\}\mu.
\end{align}
Now we have to consider two cases depending on the sign inside the curly
 brackets. If the sign is negative or zero, we have
\begin{align}
1-\frac{\delta}{2}+\sigma\eta_t-\frac{\delta}{2}\frac{\|\vw^t-W^{\ast}\|^2}{\|\vw^{t+1}-W^{\ast}\|^2}\leq 0,\nonumber
\intertext{which implies}
\label{eq:proof_fastrate_case1}
\|\vw^{t+1}-W^{\ast}\|^2\leq\frac{\delta}{2-\delta+2\sigma\eta_t}\|\vw^{t}-W^{\ast}\|^2.
\end{align}
Since $\delta\leq 3/4$, the factor in front of $\|\vw^t-W^{\ast}\|^2$ in the
 right-hand side is clearly smaller than one. We further show that this
 factor is smaller than $1/(1+\epsilon\sigma\eta_t)^2$. First we upper
 bound $\delta$ by $1/(1+\epsilon\sigma\eta_t)$ as follows:
\begin{align*}
\delta=\frac{1-\epsilon}{\sigma\eta_t}&=\frac{(1-\epsilon)(\frac{1}{\sigma\eta_t}+\epsilon)}{1+\epsilon\sigma\eta_t}\leq \frac{(1-\epsilon)\epsilon+3/4}{1+\epsilon\sigma\eta_t}\leq\frac{1}{1+\epsilon\sigma\eta_t}.
\end{align*}
Plugging the above upper bound into inequality
 \eqref{eq:proof_fastrate_case1}, we have
\begin{align*}
 \|\vw^{t+1}-W^{\ast}\|^2&\leq \frac{\delta}{1+2\sigma\eta_t}\|\vw^{t}-W^{\ast}\|^2\\
&\leq \frac{1}{(1+\epsilon\sigma\eta_t)(1+2\sigma\eta_t)}\|\vw^{t}-W^{\ast}\|^2\leq \frac{1}{(1+\epsilon\sigma\eta_t)^2}\|\vw^{t}-W^{\ast}\|^2,
\end{align*}
which completes the proof for the first case.

If on the other hand, the term inside curly brackets is positive in
\Eqref{eq:proof_fastrate_quad}, maximizing the right-hand side of
 \Eqref{eq:proof_fastrate_quad} with respect to $\mu$ gives the
 following expression:
\begin{align*}
(1-\delta)r_t&\geq 1-\frac{\delta}{2}+\sigma\eta_t-\frac{\delta}{2}r_t^2,
\intertext{where we defined $r_t:=\|\vw^{t}-W^{\ast}\|/\|\vw^{t+1}-W^{\ast}\|$. Because $r_t>0$, the
 above inequality translates into}
r_t&\geq
 \frac{\sqrt{1+2\sigma\eta_t\delta}-1+\delta}{\delta}\\
 &\geq \frac{1+\sigma\eta_t\delta-\sigma^2\eta_t^2\delta^2-1+\delta}{\delta}\\
 &\geq 1+\sigma\eta_t(1-\sigma\eta_t\delta)\\
 &=1+\epsilon\sigma\eta_t.
\end{align*}
The second line is true because for $x\geq 0 $, $\sqrt{1+x}\geq
 1+x/2-x^2/4$; the last line follows from the definition $\delta=(1-\epsilon)/(\sigma\eta_t)$.
\end{proof}

\subsection{Proof of Theorem~\ref{thm:rockafellar_to_A1}}
\label{sec:proof_rockafellar_to_A1}
\begin{proof}
Since the minimizer is unique, we denote the minimizer by $\vw^{\ast}$
 and show that the following is true:
\begin{align*}
f(\vw)-f(\vw^{\ast})\geq\sigma\|\vw-\vw^{\ast}\|^2\qquad(\forall\vw:\|\vw-\vw^{\ast}\|\leq
 c\tau L).
\end{align*}

Using Theorem X.4.2.2 in \cite{HirLem93}, for $\|\vbeta\|\leq
\tau$, we have
\begin{align}
 f^\ast(\vbeta)&\leq f^{\ast}(0)+\vbeta\T\nabla
 f^{\ast}(0)+\frac{L}{2}\|\vbeta\|^2\nonumber\\
\label{eq:fast_upper_bound}
 &=-f(\vw^\ast)+\vbeta\T\vw^\ast+\frac{L}{2}\|\vbeta\|^2,
\end{align}
where $\vw^\ast:=\argmin_{\vw\in\RR^n}f(\vw)=\nabla f^{\ast}(0)$ and
$f^\ast(0)=-f(\vw^\ast)$.

On the other hand, we have
\begin{align*}
 f(\vw)&=\sup_{\vbeta\in\RR^n}\left(\vbeta\T\vw-f^{\ast}(\vbeta)\right)\\
&\geq
 \sup_{\|\vbeta\|\leq\tau}\left(\vbeta\T\vw-f^{\ast}(\vbeta)\right)\\
&\geq
 \sup_{\|\vbeta\|\leq\tau}\left(\vbeta\T(\vw-\vw^{\ast})-\frac{L}{2}\|\vbeta\|^2\right)+f(\vw^\ast)\\
&
\begin{cases}
= f(\vw^\ast)+ \frac{1}{2L}\|\vw-\vw^{\ast}\|^2 & \textrm{(if $c\leq 1$),}\\
\geq f(\vw^\ast)+ \frac{2c-1}{2c^2L}\|\vw-\vw^{\ast}\|^2 & \textrm{(otherwise),}
\end{cases}
\end{align*}
where we used inequality \eqref{eq:fast_upper_bound} in the third line; the last
 line follows because if  $c\leq 1$, the maximum is attained at
 $\vbeta=(\vw-\vw^{\ast})/L$, and otherwise we can lower
 bound the value at the maximum by the value at
 $\vbeta=(\vw-\vw^{\ast})/(cL)$. Combining the above two cases, we have
 Theorem~\ref{thm:rockafellar_to_A1}.
\end{proof}

\bibliography{IEEEabrv,jmlr10}

\begin{sidewaystable}[htb]
 \begin{center}
  \caption{List of loss functions and their convex conjugates. Constant
  terms are ignored. For the multi-class logit loss,
  $\vz,\valpha\in\RR^{m(c-1)}$, where $m$ is the number of
  samples and $c$ is the number of classes; $y_{ik}=1$ if the $i$th
  sample belongs to the $k$th class, and zero otherwise;
  $i(k):=(i-1)c+k$ denotes the   linear index corresponding to the $k$th
  output for the $i$th sample; 
  $\delta_{i,k}$ denotes the Kronecker delta function.}
  \label{tab:loss}
\medbreak
\hspace*{-10mm}
{\footnotesize
  \begin{tabular}{p{2.2cm}|c|c|c|c|c}
   & primal loss $f_{\ell}(\vz)$ & conjugate loss $f_{\ell}^{\ast}(-\valpha)$ &
   gradient $-\nabla f_{\ell}^{\ast}(-\valpha)$ & Hessian $\nabla^2
   f_{\ell}^{\ast}(-\valpha)$ & $\gamma$\\
\hline\hline
Squared loss & $\frac{1}{2}\sum_{i=1}^n\frac{(y_i-z_i)^2}{\sigma_i^2}$ &
$\frac{1}{2}\sum_{i=1}^n\sigma_i^2(\alpha_i-y_i)^2$ &
${\rm diag}\left(\sigma_1^2,\ldots,\sigma_m^2\right)(\valpha-\vy)$ &
		   ${\rm diag}\left(\sigma_1^2,\ldots,\sigma_m^2\right)$ & $\min\limits_i\sigma_i^2$\\
\hline
Logistic loss & $\sum_{i=1}^n\log(1+\exp(-y_iz_i))$ &
$
\begin{array}{c}
\sum_{i=1}^n\bigl((\alpha_iy_i)\log(\alpha_iy_i) \\
\qquad +(1-\alpha_iy_i)\log(1-\alpha_iy_i)\bigr)
\end{array}
$ & 
$\left(y_i\log\frac{\alpha_iy_i}{1-\alpha_iy_i}\right)_{i=1}^m$ &
${\rm diag}(\frac{1}{\alpha_iy_i(1-\alpha_iy_i)})_{i=1}^m$ & 4
\\
\hline
Hyperbolic secant likelihood~{\footnotesize \citep{HauTomNolMueKaw10}} &
$\sum_{i=1}^n\log \left(e^{y_i-z_i}+e^{-y_i+z_i}\right)$ &
$\begin{array}{c}
\frac{1}{2}\sum_{i=1}^n \bigl((1-\alpha_i)\log(1-\alpha_i)\\
\qquad +(1+\alpha_i)\log(1+\alpha_i)-2\alpha_iy_i\bigr)
\end{array}$
& $(\frac{1}{2}\log\frac{1+\alpha_i}{1-\alpha_i}-y_i)_{i=1}^m$ &
${\rm diag}(\frac{1}{2(1-\alpha_i)(1+\alpha_i)})_{i=1}^m$ & 2\\
\hline
Multi-class logit {\footnotesize \citep{TomMue10}} &
$
\begin{array}{c}
\sum_{i=1}^{m}\bigl(-\sum_{k=1}^{c-1}z_{i(k)}y_{ik}\\
\quad+\log\bigl(\sum_{k=1}^{c-1}e^{z_{i(k)}}+1\bigr)\bigr)
\end{array}$ &
$
\begin{array}{c}
\sum\limits_{i=1}^{m}\bigl(\sum\limits_{k=1}^{c-1}(y_{ik}-\alpha_{i(k)})\log(y_{ik}-\alpha_{i(k)})  \\
+(y_{ic}+\sum\limits_{k=1}^{c-1}\alpha_{i(k)})\log(y_{ic}+\sum\limits_{k=1}^{c-1}\alpha_{i(k)})\bigr)\\
(0\leq y_{ik}-\alpha_{i(k)}\leq 1\\
\quad(k=1,\ldots,c-1),\\
 0\leq y_{ic}+\sum_{k=1}^{c-1}\alpha_{i(k)}\leq 1)
\end{array}$ &
$
\begin{array}{l}
\bigl(-\log\frac{y_{ik}-\alpha_{i(k)}}{y_{ic}+\sum_{k=1}^{c-1}\alpha_{i(k)}}\bigr)_{i(k)=1}^{m(c-1)} \\
\quad (i=1,\ldots,m;\\
\quad\, k=1,\ldots,c-1)
\end{array}$ &
$
\begin{array}{l}
\Bigl(\frac{\delta_{i,j}\delta_{k,l}}{y_{ik}-\alpha_{i(k)}} \\
\, +\frac{\delta_{i,j}}{y_{ic}+\sum_{k'=1}^{c-1}\alpha_{i(k')}}\Bigr)_{i(k),j(l)=1}^{m(c-1)}\\
\quad(i,j=1,\ldots,m;\\
\quad\,k,l=1,\ldots,c-1)
\end{array}$ & 1
\end{tabular}}
\bigbreak
  \caption{List of regularizers and their corresponding proximity
  operators~\eqref{eq:proximation} and the Envelope
  function~\eqref{eq:envelope}. The operation $(\cdot)_+$ is defined as
  $(x)_+:=\max(0,x)$ and applies element-wise to a matrix.}
  \label{tab:regularizers}
  \begin{tabular}{p{3cm}|c|c|c}
Description &  Regularizer  & Proximity operator $\prox{\lambda}$ &
   Envelope function $\Phi_{\lambda}^{\ast}$ \\
\hline\hline
   $\ell_1$-regularizer {\small \citep{Tib96}} &
   $\phi_{\lambda}^{\ell_1}(\vw)=\lambda\sum_{j=1}^n|w_j|$ &
   $\prox{\lambda}^{\ell_1}(\vw)=\left((|w_j|-\lambda)_+\frac{w_j}{|w_j|}\right)_{j=1}^n$ &
$\Phi_{\lambda}^\ast(\vw)=\frac{1}{2}\sum_{j=1}^n(|w_j|-\lambda)_+^2$ \\
\hline
Group lasso {\small  \citep{YuaLin06}} &
$\phi_{\lambda}^{\Gf}(\vw)=\lambda\sum_{\gf\in\Gf}\|\vw_{\gf}\|$ &
$\prox{\lambda}^{\Gf}(\vw)=\left((\|\vw_{\gf}\|-\lambda)_+\frac{\vw_{\gf}}{\|\vw_{\gf}\|}\right)_{\gf\in\Gf}$ &
$\Phi_{\lambda}^{\ast}(\vw)=\frac{1}{2}\sum_{\gf\in\Gf}(\|\vw_{\gf}\|-\lambda)_+^2$\\
\hline
Trace norm {\small \citep{FazHinBoy01,SreRenJaa05}} &
 $\phi_{\lambda}^{\rm tr}(\vw)=\lambda\sum_{j=1}^n\sigma_j(\vw)$ &
 $\prox{\lambda}^{\rm tr}(\vw)={\rm vec}\left(\mU(\mS-\lambda)_+\mV\T\right)$ &
$\Phi_{\lambda}^\ast(\vw)=\frac{1}{2}\sum_{j=1}^{r}(\sigma_j(\vw)-\lambda)_+^2$\\
\hline
Elastic-net {\small \citep{ZouHas05,TomSuz10}} &
 $\phi_{\lambda}^{\rm
 en}(\vw)=\lambda\sum_{j=1}^n((1-\theta)|w_j|+\frac{\theta}{2}w_j^2)$ &
$\prox{\lambda}^{\rm en}(\vw)=\left(\frac{(|w_j|-\lambda(1-\theta))_+}{1+\lambda\theta}\frac{w_j}{|w_j|}\right)_{j=1}^n$ &
$\Phi_{\lambda}^{\ast}(\vw)=\frac{1+\lambda\theta}{2}\sum_{j=1}^n\left(\frac{|w_j|-\lambda(1-\theta)}{1+\lambda\theta}\right)_+^2$
  \end{tabular}
 \end{center}
\end{sidewaystable}

\end{document}